\documentclass{article}
\usepackage[utf8]{inputenc}
\usepackage{amsmath, amsthm, amsfonts, amssymb, graphicx, cite, epsfig,epstopdf,color,soul}

\usepackage{enumerate}
\usepackage[shortlabels]{enumitem}
\usepackage{empheq}
\usepackage[ruled]{algorithm2e}

\usepackage{geometry}%
\usepackage[bookmarks=true,pageanchor,colorlinks,linkcolor=red,anchorcolor=blue, citecolor=blue,urlcolor=blue,hyperfootnotes=false]{hyperref}

\newcommand{\Eset}{\mathbb{E}}

\newcommand{\Rset}{\mathbb{R}}

\newcommand{\Fcal}{{\cal F}}

\newcommand{\Kcal}{{\cal K}}

\newcommand{\Ocal}{{\cal O}}

\newcommand{\Scal}{{\cal S}}

\newcommand{\Vcal}{{\cal V}}

\newcommand{\Abf}{{\bf A}}

\newcommand{\Dbf}{{\bf D}}

\newcommand{\Ibf}{{\bf I}}

\newcommand{\Pbf}{{\bf P}}

\newcommand{\Ybf}{{\bf Y}}

\newcommand{\Xhat}{{\hat{X}}}
\newcommand{\Yhat}{{\hat{Y}}}
\newcommand{\Zhat}{{\hat{Z}}}

\newtheorem{lem}{Lemma}
\newtheorem{thm}{Theorem}

\newtheorem{assump}{Assumption}
\theoremstyle{remark}
\newtheorem{remark}{Remark}

\title{Finite-Time Analysis and Restarting Scheme for Linear Two-Time-Scale Stochastic Approximation}
\author{{Thinh T. Doan\thanks{Thinh T. Doan is with the School of Industrial and Systems Engineering, Georgia Institute of Technology, GA, 30332, USA. {\tt\small thinhdoan@gatech.edu}}}}

\begin{document}
\date{}
\maketitle

\begin{abstract}
Motivated by their broad applications in reinforcement learning, we study the linear two-time-scale stochastic approximation, an iterative method using two different step sizes for finding the solutions of a system of two equations. Our main focus is to characterize the finite-time complexity of this method under time-varying step sizes and Markovian noise. In particular, we show that the mean square errors of the variables generated by the method converge to zero at a sublinear rate $\Ocal(k^{2/3})$, where $k$ is the number of iterations. We then improve the performance of this method by considering the restarting scheme, where we restart the algorithm after every predetermined number of iterations. We show that using this restarting method the complexity of the algorithm under time-varying step sizes is as good as the one using constant step sizes, but still achieving an exact converge to the desired solution. Moreover, the restarting scheme also helps to prevent the step sizes from getting too small, which is useful for the practical implementation of the linear two-time-scale stochastic approximation. 
\end{abstract}


\section{Introduction}
In this paper, we consider two-time-scale stochastic approximation ({\sf SA}), a recursive algorithm for finding the solution of a system of two equations based on simulation. In this algorithm, the first iterate is updated using step sizes that are very small compared to the ones used to update the second iterate. One can view that the update associated with the small step sizes is implemented at a ``slow" time-scale, while the other is executed at a ``fast" time-scale. In addition, the update of the ``fast" iterate depends on the ``slow" iterate and vice versa. Therefore, one needs to choose proper step sizes to guarantee the convergence of these two iterates. Indeed, an important problem in this area is to select the two step sizes so that the two iterates converge as fast as possible.

Two-time-scale {\sf SA} has received a surge of interests due to their broad applications in many areas, 
especially in reinforcement learning \cite{SBbook1998, BORKAR2005, KondaT2003, Sutton2009a,leeHe2019,KumarKR2019, borkar2008,NIPS2019_9248}. In particular, many existing algorithms for solving the important policy evaluation problem in this area are the variants of the so-called temporal difference ({\sf TD}) learning algorithms \cite{Sutton1988_TD}, and can be formulated as the two-time-scale methods. Some examples of these algorithms include {\sf TD}$(\lambda)$, gradient {\sf TD}, and target-based {\sf TD} \cite{Sutton2009a,Sutton2009b,leeHe2019}. It has been shown that the two-time-scale {\sf TD} algorithms are an important class of reinforcement learning algorithms since they are more stable under the so-called off-policy learning and perform much better than the original {\sf TD} learning in some cases \cite{Sutton2009a,Sutton2009b,leeHe2019}. In addition, the popular actor-critic method for solving the optimal policy problem in reinforcement learning is another application of the two-time-scale method, where the critic estimates the value function at a fast time-scale while the actor updates the policy parameter at a slow time-scale based on the value returned by the critic \cite{KondaT2003}. 

Other notable applications of the two-time-scale {\sf SA} 
include stochastic optimization \cite{PolyakJ1992, WangFL2017, ZhangX2019} and distributed optimization on multi-agent systems \cite{DoanMR2018b,DoanBS2017}. In these applications, it has been observed that using two-time-scale iterations one can achieve a better performance than the one-time-scale counterpart. For example, one can improve the convergence of the popular stochastic gradient descent method in optimization by considering its performance on the time-weighted average of the iterates \cite{PolyakJ1992}. In this case, another variable is used to estimate for the time average of the main variable. Moreover, two-time-scale methods have been used in distributed optimization to address the issues of communication constraints, where an additional variable is used to eliminate the errors due to imperfect communication between agents     
\cite{DoanBS2017,DoanMR2018b}.


This paper focuses on the theoretical aspect of the finite-time performance of the linear two-time-scale {\sf SA}, which includes the class of {\sf TD} learning algorithms mentioned above as a special case. In particular, our goal is to characterize the convergence rate of this method under time-varying step sizes and provide some insights on the step size selection to yield the best possible rate of the algorithm. We do it when the problem data are dependent, that is, they are sampled from Markov processes as often considered in the context of reinforcement learning. More details of the linear two-time-scale {\sf SA} under Markov samples are discussed in Section \ref{sec:2timescale} below.

\subsection{Related works}
Most of the work in the literature have focused on the celebrated {\sf SA} (a single-time-scale), introduced by Robbins and Monro \cite{RobbinsM1951}, for solving the root-finding problems under corrupted measurements of the underlying function. In particular, the most powerful and popular technique to analyze the asymptotic convergence of {\sf SA} is the Ordinary Differential Equation ({\sf ODE}) method \cite{borkar2008}. Such {\sf ODE} method shows that under the right conditions the noise effects eventually average out and the {\sf SA} iterate asymptotically follows a stable {\sf ODE}. On the other hand, the rates of convergence of {\sf SA} has been mostly considered in the context of stochastic gradient descent with i.i.d samples (the samples of the gradients of the underlying function are unbiased and i.i.d); see for example \cite{BottouCN2018} and the references therein. Motivated by a surge of recent interests in reinforcement learning, the finite-time analysis of {\sf SA} under Markov samples has been studied in \cite{Bhandari2018_FiniteTD, SrikantY2019_FiniteTD, ChenZDMC2019, HuS2019,Karimi_colt2019}

Unlike the single-time-scale {\sf SA}, the convergence properties of the two-time-scale {\sf SA} are less understood, especially its convergence rates.  The asymptotic convergence of this two-time-scale {\sf SA} can be achieved by using the {\sf ODE} methods \cite{BorkarM2000}, while its rates of convergence has been studied in \cite{KondaT2004, DalalTSM2018, DoanM2019, MokkademP2006} under i.i.d sampling and the updates are linear. The work in \cite{DalalTSM2018} provides a concentration bound for the finite-time analysis of this method, while the convergence rate has been studied in \cite{KondaT2004,DoanM2019}. On the other hand, the authors in \cite{MokkademP2006} study the convergence rate of nonlinear two-time-scale methods under i.i.d samples. Recently, its finite-time analysis has been studied in \cite{GuptaSY2019_twoscale} under constant step sizes and Markovian noise, where the authors provide the convergence rate of the iterates to  a neighborhood of the solution. The goal of this paper is to study the finite-time performance of this method under Markovian noise and time-varying step sizes, which has not been addressed in \cite{GuptaSY2019_twoscale}. In particular, due to the time-varying step sizes we use different techniques to analyze the algorithm as compared to the ones in \cite{GuptaSY2019_twoscale}. More details are provided in Section \ref{sec:finitetime} below.


\subsection{Main Contribution} 
In this paper we study the finite-time performance of the linear two-time-scale {\sf SA} under Markov samples and time-varying step sizes. In particular, we show that the mean square errors of the iterates generated by the method converge to zero at a rate $\Ocal(1/k^{2/3})$,
where $k$ is the number of the iterations. Based on our analysis we provide some insights about the choice of the two step sizes for different scenarios. Next, to improve the convergence of the method we consider a restarting scheme where the algorithm is started after a predetermined number of iterations. We show that the complexity of the two-time-scale method with restarting is the same as the one under constant step sizes, but still being able to decrease the mean square errors to zero. In addition, the restarting method also helps to prevent the time-varying step sizes from getting too small, which might be useful for the practical implementation of the two-time-scale {\sf SA}.

The remainder of this paper is organized as follows. We give a formal statement of the linear two-time-scale {SA} and its motivation in Section \ref{sec:2timescale}. The finite-time performance of this method is stated in Section \ref{sec:finitetime}, while the restarting scheme is presented in Section \ref{sec:restarting}. Finally, for an ease of exposition we provide the proofs of all technical lemmas required by our main results in Section \ref{sec:lem_proofs} and in  Appendix. 


\section{Linear two-time-scale stochastic approximation}\label{sec:2timescale}
To motivate the two-time-scale {\sf SA} method, we consider the problem of finding the solution $(X^*,Y^*)$ of the following linear system of equations 
\begin{align}
\begin{aligned}
&\Abf_{11}X^* + \Abf_{12}Y^* = b_{1}\\ 
&\Abf_{21}X^* + \Abf_{22}Y^* = b_{2},\label{prob:main}\\ 
\end{aligned}
\end{align}
where we assume that the sets of matrices $\Abf_{ij}$ and vectors $b_{i}$, for all $i,j = 1,2$ are unknown. Instead, we can only have access to their samples. Therefore, computing the solutions $(X^*,Y^*)$ through solving \eqref{prob:main} directly is impossible, motivating us to consider an alternative approach. Specifically, since we have access to the samples of $\Abf_{ij}$ and $b_{i}$ we consider the linear two-time-scale {\sf SA}, which iteratively updates an estimate $(X_{k},Y_{k})$ of $(X^*,Y^*)$ as 
\begin{align}
\begin{aligned}
X_{k+1} &= X_{k} - \alpha_k\left(\Abf_{11}(\xi_{k})X_{k} + \Abf_{12}(\xi_{k})Y_{k} - b_1(\xi_{k})\right)\\
Y_{k+1} &= Y_{k} - \beta_{k}\left(\Abf_{21}(\xi_{k})X_{k} + \Abf_{22}(\xi_{k})Y_{k} - b_{2}(\xi_{k})  \right),
\end{aligned}\label{alg:XY} 
\end{align}
where the sequence $\{\xi_{k}\}$ are the samples from a Markov process. We are interested in the case where $\beta_{k} \ll \alpha_k$, therefore, $X_{k}$ is updated at a faster time scale than $Y_{k}$. Here, we consider the noise is modeled by Markov processes, which is motivated by broad applications in reinforcement learning and machine learning. Indeed, below we provide one application of \eqref{alg:XY} in reinforcement learning and then  proceed to present our main results about its finite-time performance in the next section.     

\subsection{Motivating applications}\label{subsec:applications}
One of the main applications of the updates \eqref{alg:XY} is to study temporal difference learning algorithms in reinforcement learning with linear function approximation \cite{Sutton2009a, Sutton2009b,leeHe2019}. Specifically, one fundamental and important problem is to estimate the accumulative return rewards of a stationary policy, which is referred to as policy evaluation problems. In this context, linear two-time-scale algorithms have been used to formulate the so-called gradient temporal difference learning ({\sf GTD}) for solving policy evaluation problems in off-policy settings with linear function approximation \cite{Sutton2009a,Sutton2009b}. Indeed, let $\zeta$ be the environmental sate, $\gamma$ be the discount factor, $\phi(\zeta)$ be the feature vector of state $\zeta$, and $r$ be the reward return by the environment. Given a sequence of samples $\{\zeta_{k}\}$, one version of  {\sf GTD} are given as
\begin{align*}
X_{k+1} &= X_{k} + \alpha_{k}( \delta_{k}\phi(\zeta_{k}) - \phi(\zeta_{k})\phi^T(\zeta_{k})X_{k})\\
Y_{k+1} &= Y_{k} + \beta_{k}\left(\phi(\zeta_{k})\phi(\zeta_{k})^T-\gamma\phi(\zeta_{k})\phi(\zeta_{k+1})^T\right)X_{k},
\end{align*}
where $\delta_{k} = r_{k} + \gamma\phi(\zeta_{k+1})^TY_{k} - \phi(\zeta_{k})^TY_{k}$ is the temporal difference error and $\phi(\zeta_{k})^TY_{k}$ is the estimate of the value function at time $k$. It has been observed that the {\sf GTD} method is more stable and performs better compared to the single-time-scale counterpart (e.g., temporal difference learning) in some cases of off-policy learning for policy evaluation problems. Obviously, we can reformulate the {\sf GTD} updates above into a form of \eqref{alg:XY} with new state $\xi_{k} = (\zeta_{k},\zeta_{k+1})$ and
\begin{align*}
&\Abf_{11}(\xi_{k}) = \phi(\zeta_{k})\phi(\zeta_{k})^T,\quad \Abf_{12}(\xi_{k}) = \phi(\zeta_{k})\big(\phi(\zeta_{k})-\gamma\phi(\zeta_{k+1})\big)^T,\quad b_{1}(\xi_{k}) = r_{k}\phi(\zeta_{k})\\
&\Abf_{21}(\xi_{k}) = \big(\gamma\phi(\zeta_{k+1})-\phi(\zeta_{k})\big)\phi(\zeta_{k})^T,\qquad \Abf_{22}(\xi_{k}) = 0,\qquad b_{2}(\xi_{k}) = 0.
\end{align*}
Here the goal of the {\sf GTD} algorithm is to find the optimal parameter $Y^*$, a solution of the so-called projected Bellman equation (see \cite{Sutton2009b} for more details) and satisfying $\Eset[b_{1}(\xi_{k})] - \Eset[\Abf_{12}(\xi_{k})]Y^* = 0$. In addition, the variable $X_{k}$ is to keep track of the quantity $X^* = (\Eset[\Abf_{11}(\xi_{k})])^{-1}(\Abf_{21}^TY^* + b_{1})$. Finally, we note that the  variants of {\sf TD} learning recently studied in \cite{leeHe2019}, namely target-based {\sf TD}, can also be viewed as a version of the two-time-scale {\sf SA} in \eqref{alg:XY} under a proper formulation.   


\subsection{Main assumptions}
We introduce in this section various assumptions, which will be useful for our convergence analysis given in the next section. Our assumptions are similar to the ones considered in \cite{GuptaSY2019_twoscale}.   
\begin{assump}\label{assump:stationary}
The sequence $\{\xi_{k}\}$ is a Markov chain with state space $\Scal$. In addition, the following limits exit
\begin{align}
\lim_{k\rightarrow\infty}
\Eset[\Abf_{ij}(\xi_{k})] = \Abf_{ij}\qquad\text{and}\qquad \lim_{k\rightarrow\infty}\Eset[b_{i}(\xi_{k})] = b_{i},\quad \forall i,j = 1,2.  
\end{align}
\end{assump}
\begin{assump}\label{assump:bounded}
We assume that the matrices and vectors in \eqref{alg:XY} are uniformly bounded, i.e., for all $i,j=1,2$ and $\xi\in\Scal$ there exists a positive constant $B$ such that
\begin{align}
\begin{aligned}
&\max_{i}\;\|b_{i}(\xi)\|\leq B\qquad \text{and}\qquad\max_{i,j}\;\|A_{ij}(\xi)\|\leq \frac{1}{4}\cdot \end{aligned}\label{assump_bounded:Ineq}
\end{align} 
This also implies that the limits of these matrices and vectors are also bounded with the same constants. 
\end{assump}
\begin{assump}\label{assump:Hurwitz}
We assume that the matrices $\Abf_{11}$ and $\Delta = \Abf_{22} - \Abf_{21}\Abf_{11}^{-1}\Abf_{12}$ are positive but not necessarily symmetric, i.e., $\max\,\{\,X^T\Abf_{11} X\,,\,\;X^T\Delta X\,\} > 0$ for any vector $X$.   
\end{assump}
Finally, we consider an assumption about the mixing time of the Markov chain $\{\xi_{k}\}$.
\begin{assump}\label{assump:mix}
Given a positive constant $\alpha$, we denote by $\tau(\alpha)$ the mixing time of the Markov chain $\{\xi_{k}\}$. We assume that for all $i,j=1,2$ and $\xi\in\Scal$
\begin{align*}
&\|\Eset[A_{ij}(\xi_{k})] - \Abf_{ij}\,|\, \xi_{0} = \xi\| \leq \alpha,\quad \forall  k\geq \tau(\alpha)\\
&\|\Eset[b_{i}(\xi_{k})] - b_{i}\,|\, \xi_{0} = \xi\| \leq \alpha,\quad \forall k\geq \tau(\alpha).
\end{align*}
In addition, the Markov chain $\{\xi_{k}\}$ has a geometric mixing time, i.e., there exist a constant $C$ such that \begin{align}
\tau(\alpha) = C\log\left(\frac{1}{\alpha}\right).\label{notation:tau}
\end{align} 
\end{assump} 
We note that Assumption \ref{assump:stationary} is to guarantee the stability of the  underlying Markov chain, while Assumption \ref{assump:bounded} can be guaranteed through a proper scaling step. Indeed, in the case of policy evaluation problems with linear function approximations the matrices $\Abf_{ij}$ are defined based on the chosen feature vectors and $b_{i}$ depends on the immediate reward. In this case, one can properly rescale $\Abf_{ij}$ through feature normalization, while the reward is always assumed to be bounded \cite{Bhandari2018_FiniteTD}. Assumption \ref{assump:Hurwitz} is used to basically guarantee the existence and uniqueness of the solution $(X^*,Y^*)$ in \eqref{prob:main}. This condition is satisfied in the context of policy evaluation problems with linear function approximation \cite{Sutton2009b}. One can relax this assumption to require that the matrices $\Abf_{11}$ and $\Delta$ have complex eigenvalues with the positive real parts. Such an extension is straightforward, which we will discuss later. Finally, Assumption \ref{assump:mix} is needed in our finite-time analysis, where it states that the Markov chain $\{\xi_{k}\}$ converges to the stationary distribution exponentially fast. Note that this condition is satisfied when the underlying Markov chain is finite and ergodic \cite{Bremaud2000}.      

\subsection{Main observations}
To study the finite-time convergence of \eqref{alg:XY} we provide the main observation behind our approach. Indeed, we first reformulate the updates in \eqref{alg:XY} as
\begin{align}\label{alg:XY_reform}
\begin{aligned}
X_{k+1} &= X_{k} - \alpha_{k}(\Abf_{11}X_{k} + \Abf_{12}Y_{k} - b_{1} + \epsilon_{k})\\
Y_{k+1} &= Y_{k} - \beta_{k}(\Abf_{21}X_{k} +  \Abf_{22}Y_{k} - b_{2} + \psi_{k}),
\end{aligned}
\end{align}
where $\epsilon_k$ and $\psi_k$ are Markovian noise defined as
\begin{align}
    \begin{aligned}
&\epsilon_{k} = \Abf_{11}(\xi_{k})X_{k} + \Abf_{12}(\xi_{k})Y_{k} - b_1(\xi_{k})- \Big(\Abf_{11}X_{k} + \Abf_{12}Y_{k} - b_1\Big)\\
&\psi_{k} = \Abf_{21}(\xi_{k})X_{k} + \Abf_{22}(X_{k})Y_{k} - b_{2}(\xi_{k})-\Big(\Abf_{21}X_{k} + \Abf_{22}Y_{k} - b_2\Big). 
\end{aligned}\label{analysis:noise}
\end{align}
Here $\Abf_{ij}$ and $b_{i}$, for all $i,j = 1,2$, are given in Assumption \ref{assump:stationary}. By letting $\alpha_{k}$ and $\beta_{k}$ decrease to zero at proper rates, one can hope to asymptotically eliminate the impact of the noise while finding the solution $(X^*,Y^*)$. In addition, under Assumption \ref{assump:Hurwitz} and by Eq. \eqref{prob:main} that $(X^*,Y^*)$ satisfies
\begin{align}
\begin{aligned}
X^* &= \Abf_{11}^{-1}(b_{1}-\Abf_{12}Y^*)\\
Y^* &= (\Abf_{22}-\Abf_{21}\Abf_{11}^{-1}\Abf_{12})^{-1}(b_{2}-\Abf_{21}\Abf_{11}^{-1}b_{1}),
\end{aligned}\label{prob:X*Y*}
\end{align}
which explains Assumption \ref{assump:Hurwitz} to guarantee the existence and uniqueness of $(X^*,Y^*)$. 

Based on Eqs.\ \eqref{alg:XY_reform} and \eqref{prob:X*Y*}, our main observation is given as follows. Suppose that $X_{k}$ converges after some time $k$ and $\alpha_{k}$ decreases to zero, then by \eqref{alg:XY_reform} ideally we should have
\begin{align*}
X_{k} = \Abf_{11}^{-1}(b_{1}-\Abf_{12}Y_{k}).    
\end{align*}
Moreover, if $Y_{k}$ converges to $Y^*$ then $X_{k}$ converges to $Y^*$, which can be seen from \eqref{prob:X*Y*}. Thus, to study the convergence of the linear two-time-scale {\sf SA} \eqref{alg:XY_reform}, it is equivalent to consider the convergence of the follow residual variables $\Xhat_{k},\Yhat_{k}$ to zero
\begin{align}
\begin{aligned}
&\hat{X}_{k} = X_{k} - \Abf_{11}^{-1}( b_{1} - \Abf_{12}Y_{k})\\
&\Yhat_{k} = Y_{k} - Y^*.
\end{aligned}\label{alg:XY_hat}
\end{align}
Indeed, the rest of this paper aims to study the rate of  convergence of $\|\Xhat_{k}\|^2$ and $\|\Yhat_{k}\|^2$ to zero in expectation. Moreover, as will be seen in the next section, introducing such residual variables helps us to facilitate our analysis. Such an observation was considered in \cite{KondaT2004}. However, while an asymptotic convergence rate was provided under i.i.d noise, we provide here a finite-time analysis for the convergence of the linear two-time-scale methods under Markovian noise.    


\section{Finite-time error bounds}\label{sec:finitetime}
In this section, we present the main results of this paper, where we provide a finite-time error bound for the convergence of the mean squared errors of the residual variables in  \eqref{alg:XY_hat}. Our result basically states that the mean square errors converge to zero at a rate $\mathcal{O}(1/k^{2/3})$ where $k$ is the number of iterations.  Our analysis also gives some insights about the choice of the two step sizes for different scenarios, which might also be useful for practical implementation. More details of step size selection are given later.     

We start our analysis by introducing a bit more notation. Recall that $\Abf_{11}$ and $\Delta$ satisfy Assumption \ref{assump:Hurwitz}, that is, they are positive. We denote by $0 < \gamma$ and $0 < \rho$ the smallest eigenvalues of $\Abf_{11}$ and $\Delta$, respectively. In addition, let $\lambda_{1}\leq\ldots\leq \lambda_{n}$ be the singular values of $\Abf_{11}$ and $ \sigma_{1}\leq\ldots\leq \sigma_{n}$ be the singular values of $\Delta$. Moreover, let $\Kcal^*$ be a positive integer such that
\begin{align}
\sum_{t=k-\tau(\alpha_{k})}^{k}\alpha_{t}\leq \tau(\alpha_{k})\alpha_{k-\tau(\alpha_{k})}\leq \log(2),\qquad \forall k\geq \Kcal^*,   \label{notation:K1*}
\end{align}
where recall that $\tau(\alpha_{k})$ is the mixing time defined in Assumption \ref{assump:mix} associated with the step size $\alpha_{k}$. Note that such a positive integer $\Kcal^*$ exists since $\alpha_{k}$ is chosen to be nonincreasing and decreasing to zero, and $\tau(\alpha_{k}) = C\log(1/\alpha_{k})$ given in \eqref{notation:tau} implying
\begin{align*}
    \lim_{k\rightarrow\infty}\tau(\alpha_{k})\alpha_{k} = 0.
\end{align*}
Finally, we consider the following Lyapunov function $V$, which takes into account the coupling between the two variables and step sizes, 
\begin{align}
V_{k} = \Eset\left[\|\Yhat_{k}\|^2\right] + \frac{1}{2\gamma\rho}\frac{\beta_{k}}{\alpha_{k}}\Eset\left[\|\Xhat_{k}\|^2\right].    \label{notation:V}
\end{align}
We now ready to state the main result of our paper, which is the rate of convergence of $\|\Xhat_{k}\|^2$ and $\|\Yhat_{k}\|^2$ in expectation, in the following theorem. The analysis of this result is presented in Section \ref{subsec:thm_analysis}
\begin{thm}\label{thm:main} 
Suppose that Assumptions \ref{assump:stationary}--\ref{assump:mix} hold. Let $\{X_{k},Y_{k}\}$ be generated by \eqref{alg:XY} with $X_{0}$ and $Y_{0}$ initialized arbitratily. Let $\{\alpha_{k},\beta_{k}\}$ be two sequences of nonnegative and nonincreasing step sizes satisfying 
\begin{align}
\begin{aligned}
&\frac{\beta_{0}}{\alpha_{0}} \leq \max\left\{2\gamma\rho,\frac{\gamma}{2\rho}\right\},\qquad \beta_{0}\geq \frac{1}{\rho}\\
&\sum_{k=0}^{\infty}\alpha_{k} =  \sum_{k=0}^{\infty} \beta_{k} = \infty,\quad \sum_{k=0}^{\infty}\left(\tau(\alpha_{k})\alpha_{k-\tau(\alpha_{k})}\alpha_{k} + \beta_{k}^2 +\alpha_{k}^2 + \frac{\beta_{k}^2}{\alpha_{k}}\right) \leq C_{0} < \infty, 
\end{aligned}\label{thm_rate:stepsizes}
\end{align}
where $C_0$ is some positive constant. Moreover, we denote by $C_{1},C_{2}$ positive constants
\begin{align}
\begin{aligned}
C_{1} &= \left(\Eset[\|\Zhat_{0}\|^2] + \frac{38C_{0}(1+8\lambda_{1})^5(2B+\|Y^*\|)^2}{\lambda_{1}^8}\right) e^{\frac{60C_{0}(\gamma+1)(8\lambda_{1}+1)^{5}(1+\alpha_{0})^2}{\gamma\lambda_{1}^{6}}}\\
C_{2} &= \left(2(13\gamma\rho+3)(2B+\|Y^*\|)^2 + \frac{9C_{1}(3+7\gamma\rho)}{4\rho\gamma}\right)\frac{(8\lambda_{1}+1)^{5}}{\lambda_{1}^{8}}\cdot
\end{aligned}\label{thm_rate:constants}
\end{align}
Then, we have for all $k\geq\Kcal^*$
\begin{align}
V_{k+1} &\leq (1-\rho\beta_{k})V_{k} + \frac{2C_{1}(1+\alpha_{0})^2\beta_{k}^3}{\rho\gamma^2\lambda_{1}^2\alpha_{k}^2} + C_{2}\left(\tau(\alpha_{k})\alpha_{k-\tau(\alpha_{k})}\beta_{k} + \beta_{k}^2 + \alpha_{k}\beta_{k}\right). \label{thm_rate:Ineq1}
\end{align}
In addition, let $\beta_{k} = \beta_{0}/(k+1)$, and $\alpha_{k} = \alpha_{0}/(k+1)^{2/3}$, we obtain for $k\geq \Kcal^*$
\begin{align}
V_{k+1}&\leq \frac{\Kcal^*V_{\Kcal^*}}{k+1} + \left(\frac{8C_{1}(1+\alpha_{0})^2\beta_{0}^3}{\rho\gamma^2\lambda_{1}^2\alpha_{0}^2}+\frac{3C_2\alpha_{0}\beta_{0}}{2}\right)\frac{1}{(k+1)^{2/3}}\notag\\ 
&\qquad + \frac{C_{2}\beta_{0}^2(1+\log(k+1))}{k+1} + \frac{3CC_{2}\beta_{0}\log^2(k+1)}{k+1},     \label{thm_rate:Ineq2}
\end{align}
where the constant $C$ is defined in \eqref{notation:tau}.
\end{thm}
\begin{remark}\label{remark:thm_rate}
We first make some comments about the upper bound in Eq.\ \eqref{thm_rate:Ineq2}
\begin{align*}
V_{k+1} \leq \Ocal\left(\frac{V_{\Kcal^*}}{k+1}\right) +    \Ocal\left(\frac{\beta_{0}^{3}}{\alpha_{0}^2(k+1)^{2/3}}\right) + \Ocal\left(\frac{C\log^2(k+1)}{k+1}\right).
\end{align*}
The first term shows the dependence on the initial conditions while the last term shows the dependence on the mixing time through the factor $C\log^2(k+1)$. Both of these impacts decay to zero at a rate $\tilde{\Ocal}(1/(k+1))$ as we would expect. On the other hand, the second term shows the coupling between the slow and fast iterates through the ratio $\Ocal(\beta_k/\alpha_k)$. This partially explains the rate $\Ocal(1/k^{2/3})$ for some specific choice of $\alpha_{k}$ and $\beta_{k}$. We discuss more details about the step size selection in the next subsection.  

Second, our convergence rate in \eqref{thm_rate:Ineq2} is the same as the one studied in \cite{KondaT2004, DoanM2019}. However, our result is fundamentally different from the one studied in \cite{KondaT2004} since they provide an asymptotic rate under i.i.d noise. On the other hand, we study finite-time error bounds of the iterates at every iteration $k\geq0$ under Markovian noise. Our result is an extension of the ones in \cite{DoanM2019}, where the authors consider i.i.d samples.   

Finally, in the context of policy evaluation problems in reinforcement learning presented in Section \ref{subsec:applications}, our result shows that the {\sf GTD} algorithm converges in expectation at a rate $\mathcal{O}(1/k^{2/3})$. In \cite{DalalTSM2018} the authors show that this method converges at a rate $\Ocal(k^{-1/3+\kappa/3})$ with high probability for $\kappa\in(0,1)$.   
\end{remark}
\subsection{Step size selection}
In Eq.\ \eqref{thm_rate:Ineq1} we show the impacts of the two step sizes on the performance of the two-time-scale {\sf SA}. One can use this upper bound to choose the two step sizes for different applications as long as they satisfy our conditions \eqref{thm_rate:stepsizes}. One choice of these step sizes $\{\alpha(k),\beta(k)\}$ can be given as
\begin{align*}
\beta(k) = \frac{\beta_{0}}{k+1},\quad \alpha(k) = \frac{\alpha_{0}}{(k+1)^s},\quad \forall s\in\left(\frac{1}{2}\,,\,1\right).
\end{align*}
In addition, we refer to the first term $\Vcal_{\Kcal}^*$ on the right-hand side of Eq.\ \eqref{thm_rate:Ineq2} as the ``bias" since it depends on the initial conditions. Similarly, we call the other terms as the ``variance" in the updates. One can choose the step size $\beta_{k}$ as large as possible, e.g., $\beta_{k} = 1/k^{3/4}$, to quickly eliminate the bias in \eqref{thm_rate:Ineq1}. However, the larger $\beta_{k}$ the slower the variance decays to $0$, as can be seen from the ratio $\beta_{k}^{3}/\alpha_{k}^2$. In general, one needs to balance these two step sizes. Since the mixing time $\tau(\alpha_{k}) = \Ocal(\log(1/\alpha_{k})$ much smaller than the two step sizes it can be ignored here. Thus, using the variance term one can choose the step sizes $\alpha_{k},\beta_{k}$ such that
\begin{align*}
\alpha_{k}\beta_{k} = \frac{\beta_{k}^3}{\alpha_{k}^2} \Rightarrow \alpha_{k}^3 = \beta_{k}^2,    
\end{align*}
which together with the bias term yields our choice in deriving Eq.\ \eqref{thm_rate:Ineq2}.


Finally, under constant step sizes, i.e., $\beta_{k} = \beta$ and $\alpha_{k} = \alpha$ with some proper choice of $\alpha,\beta$, we recover the results studied in \cite{GuptaSY2019_twoscale}. In this case, we have $V_{k}$ decays exponentially fast to a ball surrounding the origin, where the size of the ball depends on the step sizes $\alpha,\beta$.     

\subsection{Proof of Theorem \ref{thm:main}}\label{subsec:thm_analysis}
We now provide the analysis of our main results in Theorem \ref{thm:main}. We first present the following three key lemmas, where their proofs are presented in Section \ref{sec:lem_proofs} for convenience. Recall that $\gamma$ and $\rho$ are the smallest eigenvalues of $\Delta_{11}$ and $\Delta$, respectively. Also, $\lambda_{1}\leq\ldots\leq \lambda_{n}$ and $\sigma_{1}\leq \ldots\leq\sigma_{n}$ are the singular values of $\Abf_{11}$ and $\Delta$, respectively.  Finally, we denote by $\Zhat = [\Xhat^T, \Yhat^T]^T$. 
\begin{lem}\label{lem:Xhat2}
Consider the assumptions in Theorem \ref{thm:main}. Then for all $k\geq \Kcal_{1}^*$
\begin{align}
\Eset\left[\|\Xhat_{k+1}\|^2\right]&\leq (1-\gamma\alpha_{k})\Eset\left[\|\Xhat_{k}\|^2\right] + \frac{\alpha_{k}^2}{16}\Eset\left[\|\Xhat_{k}\|^2\right] + \frac{(1+\alpha_{0})^2(1+\sigma_{n})^2\beta_{k}^2}{\gamma\lambda_{1}^2\alpha_{k}}\Eset\left[\|\Zhat_{k}\|^2\right]\notag\\
&\quad + \frac{6(1+\sigma_{n})^2(8\lambda_{1}+1)^{5}}{\lambda_{1}^{6}}\left(\tau(\alpha_{k})\alpha_{k-\tau(\alpha_{k})}\alpha_{k} + \alpha_{k}^2+\alpha_{k}\beta_{k}\right)\Eset\left[\|\Zhat_{k}\|^2\right]\notag\\
&\quad + \frac{6(1+4\sigma_{n})(1+8\lambda_{1})^5(2B+\|Y^*\|)^2}{\lambda_{1}^8}\left(\tau(\alpha_{k})\alpha_{k-\tau(\alpha_{k})}\alpha_{k}+\alpha_{k}^2\right).\label{lem_Xhat2:Ineq}
\end{align}
\end{lem}

\begin{lem}\label{lem:Yhat2}
Consider the assumptions in Theorem \ref{thm:main}. Then for all $k\geq \Kcal_{1}^*$
\begin{align}
\Eset\left[\|\Yhat_{k+1}\|^2\right]
&\leq (1-\rho\beta_{k})\Eset\left[\|\Yhat_{k}\|^2\right] + \frac{\beta_{k}}{4\rho}\Eset\left[\|\Xhat_{k}\|^2\right]\notag\\
&\quad  + \frac{7(2\sigma_{n}+1)^2(8\lambda_{1}+1)^{5}}{\lambda_{1}^{5}}\Big(\tau(\alpha_{k})\alpha_{k-\tau(\alpha_{k})}\beta_{k} + \beta_{k}^2\Big)\Eset\left[\|\Zhat_{k}\|^2\right]\notag\\
&\quad + \frac{13(4\sigma_{n}+1)(8\lambda_{1}+1)^5(2B+\|Y^*\|)^2}{
\lambda_{1}^7}\Big(\tau(\alpha_{k})\alpha_{k-\tau(\alpha_{k})}\beta_{k}+\beta_{k}^2\Big).\label{lem_Yhat2:Ineq}
\end{align}
\end{lem}

\begin{lem}\label{lem:Zhat2_bounded}
Consider the assumptions in Theorem \ref{thm:main}. Given $C_{1}$ in \eqref{thm_rate:constants} we have 
\begin{align}
\Eset[\|\Zhat_{k+1}\|^2] \leq C_{1}.\label{lem_Zhat2_bounded:Ineq}
\end{align}
\end{lem}

With these preliminary results, we now proceed to show Theorem \ref{thm:main}. 

\begin{proof}[Proof of Theorem \ref{thm:main}]
Let $\omega_{k}$ be defined as
\begin{align*}
\omega_{k} =     \frac{1}{2\gamma\rho}\frac{\beta_{k}}{\alpha_{k}}\cdot
\end{align*}
First, since $\beta_{k}/\alpha_{k}$ is nonincreasing, multiplying both sides of Eq. \eqref{lem_Xhat2:Ineq} by $\omega_{k}$ gives
\begin{align}
&\omega_{k+1}\Eset\left[\|\Xhat_{k+1}\|^2\right] \leq \omega_{k}\Eset\left[\|\Xhat_{k+1}\|^2\right] =  \frac{1}{2\gamma\rho}\frac{\beta_{k}}{\alpha_{k}}\Eset\left[\|\Xhat_{k+1}\|^2\right]\notag\\
&\stackrel{\eqref{lem_Xhat2:Ineq}}{\leq} \omega_{k}\Eset\left[\|\Xhat_{k}\|^2\right] -\frac{\beta_{k}}{2\rho}\Eset\left[\|\Xhat_{k}\|^2\right] +  \frac{\beta_{k}\alpha_{k}}{32\gamma\rho}\Eset\left[\|\Xhat_{k}\|^2\right] + \frac{(1+\alpha_{0})^2(1+\sigma_{n})^2\beta_{k}^3}{2\rho\gamma^2\lambda_{1}^2\alpha_{k}^2}\Eset\left[\|\Zhat_{k}\|^2\right]  \notag\\
&\qquad + \frac{3(1+\sigma_{n})^2(8\lambda_{1}+1)^{5}}{\gamma\rho\lambda_{1}^{6}}\left(\tau(\alpha_{k})\alpha_{k-\tau(\alpha_{k})}\beta_{k} + \alpha_{k}\beta_{k} + \beta_{k}^2\right)\Eset\left[\|\Zhat_{k}\|^2\right]\notag\\
&\qquad + \frac{3(1+4\sigma_{n})(1+8\lambda_{1})^5(2B+\|Y^*\|)^2}{\gamma\rho\lambda_{1}^8}\left(\tau(\alpha_{k})\alpha_{k-\tau(\alpha_{k})}\beta_{k} + \alpha_{k}\beta_{k}\right)\allowdisplaybreaks\notag\\
&= (1-\rho\beta_{k})\omega_{k}\Eset\left[\|\Xhat_{k}\|^2\right] + \frac{\beta_{k}^2}{2\gamma\alpha_{k}} \Eset\left[\|\Xhat_{k}\|^2\right]  -\frac{\beta_{k}}{2\rho}\Eset\left[\|\Xhat_{k}\|^2\right] + \frac{(1+\alpha_{0})^2(1+\sigma_{n})^2\beta_{k}^3}{2\rho\gamma^2\lambda_{1}^2\alpha_{k}^2}\Eset\left[\|\Zhat_{k}\|^2\right]  \notag\\
&\qquad + \frac{3(1+\sigma_{n})^2(8\lambda_{1}+1)^{5}}{\gamma\rho\lambda_{1}^{6}}\left(\tau(\alpha_{k})\alpha_{k-\tau(\alpha_{k})}\beta_{k} + \alpha_{k}\beta_{k} + \beta_{k}^2\right)\Eset\left[\|\Zhat_{k}\|^2\right]\notag\\
&\qquad + \frac{3(1+4\sigma_{n})(1+8\lambda_{1})^5(2B+\|Y^*\|)^2}{\gamma\rho\lambda_{1}^8}\left(\tau(\alpha_{k})\alpha_{k-\tau(\alpha_{k})}\beta_{k} + \alpha_{k}\beta_{k}\right).\label{thm_rate:Eq1a}
\end{align}
By \eqref{thm_rate:stepsizes} we have $\beta_{k}/\alpha_{k}\leq \gamma/2\rho$. Thue we obtain
\begin{align*}
-\frac{\beta_{k}}{2\rho} + \frac{\beta_{k}^2}{2\gamma\alpha_{k}} + \frac{\beta_{k}}{4\rho}  \leq  \beta_{k}\left(-\frac{1}{2\rho} + \frac{1}{4\rho} + \frac{1}{4\rho}\right) = 0.
\end{align*}
Thus, using the preceding relation and adding Eq.\ \eqref{thm_rate:Eq1a} to Eq.\ \eqref{lem_Yhat2:Ineq} yields
\begin{align*}
&V_{k+1} = \Eset\left[\|\Yhat_{k+1}\|^2\right] + \omega_{k+1}\Eset\left[\|\Xhat_{k+1}\|^2\right]\notag\\
&\leq (1-\rho\beta_{k})V_{k}  - \frac{\beta_{k}}{2\rho}\Eset\left[\|\Xhat_{k}\|^2\right] + \frac{\beta_{k}^2}{2\gamma\alpha_{k}} \Eset\left[\|\Xhat_{k}\|^2\right]  + \frac{\beta_{k}}{4\rho}\Eset\left[\|\Xhat_{k}\|^2\right]\notag\\ 
&\quad + \frac{(1+\alpha_{0})^2(1+\sigma_{n})^2\beta_{k}^3}{2\rho\gamma^2\lambda_{1}^2\alpha_{k}^2}\Eset\left[\|\Zhat_{k}\|^2\right]  \notag\\
&\quad + \frac{3(1+\sigma_{n})^2(8\lambda_{1}+1)^{5}}{\gamma\rho\lambda_{1}^{6}}\left(\tau(\alpha_{k})\alpha_{k-\tau(\alpha_{k})}\beta_{k} + \alpha_{k}\beta_{k} + \beta_{k}^2\right)\Eset\left[\|\Zhat_{k}\|^2\right]\notag\\
&\quad + \frac{3(1+4\sigma_{n})(1+8\lambda_{1})^5(2B+\|Y^*\|)^2}{\gamma\rho\lambda_{1}^8}\left(\tau(\alpha_{k})\alpha_{k-\tau(\alpha_{k})}\beta_{k} + \alpha_{k}\beta_{k}\right)\notag\\
&\quad + \frac{7(2\sigma_{n}+1)^2(8\lambda_{1}+1)^{5}}{\lambda_{1}^{5}}\Big(\tau(\alpha_{k})\alpha_{k-\tau(\alpha_{k})}\beta_{k} + \beta_{k}^2\Big)\Eset\left[\|\Zhat_{k}\|^2\right]\notag\\
&\quad + \frac{13(4\sigma_{n}+1)(8\lambda_{1}+1)^5(2B+\|Y^*\|)^2}{
\lambda_{1}^7}\Big(\tau(\alpha_{k})\alpha_{k-\tau(\alpha_{k})}\beta_{k}+\beta_{k}^2\Big)\allowdisplaybreaks\notag\\
&\leq (1-\rho\beta_{k})V_{k}  + \frac{(1+\alpha_{0})^2(1+\sigma_{n})^2\beta_{k}^3}{2\rho\gamma^2\lambda_{1}^2\alpha_{k}^2}\Eset\left[\|\Zhat_{k}\|^2\right]  \notag\\
&\quad + \frac{(3+7\gamma\rho)(2\sigma_{n}+1)^2(8\lambda_{1}+1)^{5}}{\rho\gamma\lambda_{1}^{6}}\left(\tau(\alpha_{k})\alpha_{k-\tau(\alpha_{k})}\beta_{k} + \beta_{k}^2 + \alpha_{k}\beta_{k}\right)\Eset\left[\|\Zhat_{k}\|^2\right]\notag\\
&\quad + \frac{(13\gamma\rho+3)(4\sigma_{n}+1)(8\lambda_{1}+1)^5(2B+\|Y^*\|)^2}{
\lambda_{1}^8}\Big(\tau(\alpha_{k})\alpha_{k-\tau(\alpha_{k})}\beta_{k}+\beta_{k}^2+\alpha_{k}\beta_{k}\Big).
\end{align*}
By using Eq.\ \eqref{lem_Zhat2_bounded:Ineq}, the preceding relation gives Eq.\ \eqref{thm_rate:Ineq1}, i.e.,
\begin{align*}
&V_{k+1}\leq (1-\rho\beta_{k})V_{k} + \frac{(1+\alpha_{0})^2(1+\sigma_{n})^2\beta_{k}^3}{2\rho\gamma^2\lambda_{1}^2\alpha_{k}^2}C_{1} \notag\\
&\qquad + \frac{(3+7\gamma\rho)(2\sigma_{n}+1)^2(8\lambda_{1}+1)^{5}}{\rho\gamma\lambda_{1}^{6}}\left(\tau(\alpha_{k})\alpha_{k-\tau(\alpha_{k})}\beta_{k} + \beta_{k}^2 + \alpha_{k}\beta_{k}\right)C_{1}\notag\\
&\qquad + \frac{(13\gamma\rho+3)(4\sigma_{n}+1)(8\lambda_{1}+1)^5(2B+\|Y^*\|)^2}{
\lambda_{1}^8}(\tau(\alpha_{k})\alpha_{k-\tau(\alpha_{k})}\beta_{k}+\beta_{k}^2+\alpha_{k}\beta_{k})\notag\\
&\leq (1-\rho\beta_{k})V_{k} + \frac{2C_{1}(1+\alpha_{0})^2\beta_{k}^3}{\rho\gamma^2\lambda_{1}^2\alpha_{k}^2} + C_{2}\left(\tau(\alpha_{k})\alpha_{k-\tau(\alpha_{k})}\beta_{k} + \beta_{k}^2 + \alpha_{k}\beta_{k}\right),
\end{align*}
where $C_2$ is defined in \eqref{thm_rate:constants} and we use Assumption \ref{assump:bounded} to have $\sigma_{n}\leq 1/4$. Next, we consider the choice of $\beta_{k} = \beta_{0}/(k+2)$ and $\alpha_{k} = \alpha_{0}/(k+1)^{2/3}$ for some $\beta_{0}\geq 1/\rho$. Since $\beta_{0}\geq 1/\rho$ we have
\begin{align*}
1 - \rho\beta_{k}\leq 1 - \frac{\rho}{\rho(k+1)} = \frac{k}{k+1},   
\end{align*}
which when substituting into Eq.\ \eqref{thm_rate:Ineq1} yields for all $k\geq \Kcal^*$
\begin{align}
&V_{k+1} \leq \frac{k}{k+1} V_{k} + \frac{2C_{1}(1+\alpha_{0})^2\beta_{k}^3}{\rho\gamma^2\lambda_{1}^2\alpha_{k}^2} + C_{2}\left(\tau(\alpha_{k})\alpha_{k-\tau(\alpha_{k})}\beta_{k} + \beta_{k}^2 + \alpha_{k}\beta_{k}\right)\notag\\
&\leq \frac{\Kcal^*V_{\Kcal^*}}{k+1} + \frac{2C_{1}(1+\alpha_{0})^2}{\rho\gamma^2\lambda_{1}^2}\sum_{t=\Kcal^*}^{k}\frac{\beta_{k}^3}{\alpha_{k}^2}\prod_{\ell=t+1}^{k}\frac{\ell}{\ell+1}\notag\\
&\qquad + C_{2}\sum_{t=\Kcal^*}^{k}\left(\tau(\alpha_{t})\alpha_{t-\tau(\alpha_{t})}\beta_{t} + \beta_{t}^2 + \alpha_{t}\beta_{t} \right)\prod_{\ell=t+1}^{k}\frac{\ell}{\ell+1}\notag\\
&= \frac{\Kcal^*V_{\Kcal^*}}{k+1} + \frac{2C_{1}(1+\alpha_{0})^2}{\rho\gamma^2\lambda_{1}^2}\sum_{t=\Kcal^*}^{k}\frac{\beta_{k}^3}{\alpha_{k}^2}\frac{t+1}{k+1} + C_{2}\sum_{t=\Kcal^*}^{k}\left(\tau(\alpha_{t})\alpha_{t-\tau(\alpha_{t})}\beta_{t} + \beta_{t}^2 + \alpha_{t}\beta_{t} \right)\frac{t+1}{k+1}\cdot\label{thm_rate:Eq1} 
\end{align}
Using the integral test and $\tau(\alpha_{k}) = C\log(1/\alpha_{k})$ we consider 
\begin{align*}
1.\quad &\sum_{t=\Kcal^*}^{k}\tau(\alpha_{t})\beta_{t}\alpha_{t-\tau(\alpha_{t})}(t+1) \leq \alpha_{0}\beta_{0}\sum_{t=\Kcal^*}^{k}\frac{C\log((t+1)^{2/3})}{(t+1)^{2/3}-C\log((t+1)^{2/3})}\notag\\
&\qquad\leq 2\alpha_{0}\beta_{0} \sum_{t=\Kcal^*}^{k}\frac{C\log((t+1)^{2/3})}{(t+1)^{2/3}}\notag\\
&\qquad \leq 2C\alpha_{0}\beta_{0}\frac{\log(\Kcal^*+1)^{2/3}}{(\Kcal^*+1)^{2/3}} +  2C\alpha_{0}\beta_{0}\int_{\Kcal^*}^{k} \frac{\log(t+1)^{2/3}}{(t+1)^{2/3}}dt\notag\\
&\qquad\leq 2C\alpha_{0}\beta_{0} +  C\alpha_{0}\beta_{0}\log^2(k+1)\leq 3C\alpha_{0}\beta_{0}\log^2(k+1).\\
2.\quad &\sum_{t=\Kcal^*}^{k}\beta_{t}^2(t+1) \leq \sum_{t=\Kcal^*}^{k}\frac{\beta_{0}^2}{t+1}\leq \beta_{0}^2\big(1+\log(k+1)\big).\\
3.\quad &\sum_{t=\Kcal^*}^{k}\frac{\beta_{t}^3(t+1)}{\alpha_{t}^2} = \frac{\beta_{0}^2}{\alpha_{0}^2}\sum_{t=\Kcal^*}^{k}\frac{1}{(t+1)^{2/3}}\leq \frac{4\beta_{0}^2(k+1)^{1/3}}{\alpha_{0}^2}\cdot\notag\\
4.\quad &\sum_{t=\Kcal^*}^{k}\alpha_{t}\beta_{t}(t+1)= \beta_{0}\alpha_{0}\sum_{t=\Kcal^*}^{k}\frac{1}{(t+1)^{2/3}}\leq \frac{3\alpha_{0}\beta_{0}(k+1)^{1/3}}{2}\cdot
\end{align*}
Applying these relations into Eq.\ \eqref{thm_rate:Eq1} immediately gives us Eq.\ \eqref{thm_rate:Ineq2}.
\end{proof}

\begin{remark}
One can relax Assumption \ref{assump:Hurwitz} to only require that the matrices $\Abf_{11}$ and $\Delta$ have complex eigenvalues with positive real parts. Indeed, there exist two real positive definite matrices $\Pbf_{X}$ and $\Pbf_{Y}$ satisfying  the Lyapunov equation
\begin{align*}
\Ibf &=  \Abf_{11}^T\Pbf_{X} + \Pbf_{X}\Abf_{11}\\    
\Ibf &=  \Delta^T\Pbf_{Y} + \Pbf_{Y}\Delta.    
\end{align*}
Thus, we only need to replace the norms $\|\Xhat_{k}\|^2$ and $\|\Yhat_{k}\|^2$ by the weighted norms $\|\Xhat_{k}\|_{\Pbf_{X}}^2$ and $\|\Yhat_{k}\|_{\Pbf_{Y}}^2$, respectively, where $\|X\|_{\Pbf}^2 = X^T\Pbf X$.   
\end{remark}


\section{Restarting Two-Time-Scale {\sf SA}}\label{sec:restarting}
In this section, we improve the convergence of the linear two-time-scale {\sf SA} studied in Theorem \ref{thm:main} by adopting the restarting method from optimization literature; see for example \cite{Lan2016_GSC}. The main motivation of this method can be explained as follows. By Eq.\ \eqref{thm_rate:Ineq2} and since $\log^2(k+1)\leq (k+1)^{1/3}$ for all $k\geq0$, we obtain
\begin{align}
V_{k+1} \leq \frac{\Kcal^*V_{\Kcal^*}}{k+1} + 
\left(\frac{8C_{1}(1+\alpha_{0})^2\beta_{0}^3}{\rho\gamma^2\lambda_{1}^2\alpha_{0}^2}+\frac{\beta_{0}C_2(\alpha_{0}+2\beta_{0} + 6C)}{2}\right)\frac{1}{(k+1)^{2/3}}.\label{sec_restarting:Ineq}
\end{align}
As we mention in Remark \ref{remark:thm_rate}, the first term of \eqref{sec_restarting:Ineq} represents for the bias and the second term is the variance of the Markovian noise. While there is a little hope to improve the convergence of the variance, we can improve the convergence of the bias, which depends on the initial condition of our method. In addition, the convergence of the bias depends on the mixing time (transient time) of the Markov chain, which has geometric convergence due to Assumption \ref{assump:mix}. Thus, we should expect that this bias should decay to zero exponentially. However, due to our choice of the time-varying step sizes we only have a sublinear rate as shown in Eq.\ \eqref{sec_restarting:Ineq}, that is, when the step sizes become small the convergence rate of the bias and the variance are the same. To address this issue we present below a restarting scheme, where we restart the two-time-scale {\sf SA} whenever the rates of the bias and the variance are equal. This will help us to keep the step sizes from becoming small, therefore, improving the convergence of the bias. To do that, we first consider the following result to quantify the Lyapunov function $V_k$ for $k\leq \Kcal^*$, the transient time of the two-time-scale updates. The proof of this lemma is presented in Appendix \ref{apx:Proof_lem_V0}.    

\begin{lem}\label{lem:V0}
Suppose that Assumption \ref{assump:bounded} holds. Let $\{X_{k},Y_{k}\}$ be generated by \eqref{alg:XY} with $X_{0}$ and $Y_{0}$ initialized arbitratily. Let $\alpha_{k} = \alpha_{0}/(k+1)^{2/3}$ and $\beta_{k} = \beta_{0}/(k+1)$ satisfying \eqref{thm_rate:stepsizes}. Given $\Kcal^*$ in \eqref{notation:K1*} we have
\begin{align}
V_{\Kcal^*} \leq \frac{8(\beta_{0}+\gamma\rho\alpha_{0})(1+\alpha_{0})^{2\Kcal^*}}{\beta_{0}\lambda_{1}^2}V_{0} +  \frac{25(B+\|Y^*\|)^2}{\lambda_{1}^6}(1+\alpha_{0})^{2\Kcal^*}.\label{lem_V0:Ineq}
\end{align}
\end{lem}
Next, using Eq.\ \eqref{lem_V0:Ineq} into Eq.\ \eqref{sec_restarting:Ineq} we have
\begin{align}
V_{k}&\leq \frac{8(\beta_{0}+\gamma\rho\alpha_{0})\Kcal^*(1+\alpha_{0})^{2\Kcal^*}V_{0}}{\beta_{0}\lambda_{1}^2 k} + \frac{25\Kcal^*(B+\|Y^*\|)^2(1+\alpha_{0})^{2\Kcal^*}}{\lambda_{1}^6k}\notag\\  
&\qquad + \left(\frac{8C_{1}(1+\alpha_{0})^2\beta_{0}^3}{\rho\gamma^2\lambda_{1}^2\alpha_{0}^2}+\frac{\beta_{0}C_2(\alpha_{0}+2\beta_{0} + 6C)}{2}\right)\frac{1}{(k+1)^{2/3}}\notag\\
&\leq \frac{\Psi_{1}V_{0}}{k} + \frac{\Psi_{2}}{k^{2/3}},\label{sec_restarting:Ineq1}
\end{align}
where $\Psi_{1}$ and $\Psi_{2}$ are defined as
\begin{align}
\Psi_{1} &=   \frac{8(\beta_{0}+\gamma\rho\alpha_{0})\Kcal^*(1+\alpha_{0})^{2\Kcal^*}}{\beta_{0}\lambda_{1}^2}\notag\\
\Psi_{2} &=   \frac{25\Kcal^*(B+\|Y^*\|)^2(1+\alpha_{0})^{2\Kcal^*}}{\lambda_{1}^6} + \frac{8C_{1}(1+\alpha_{0})^2\beta_{0}^3}{\rho\gamma^2\lambda_{1}^2\alpha_{0}^2}+\frac{\beta_{0}C_2(\alpha_{0}+2\beta_{0} + 6C)}{2}\cdot\label{restarting:Psi}
\end{align}
Thus, to find a solution such that $V_{k}\leq \epsilon$ the total number of iteration required by the two-time-scale method is at most
\begin{align}
K = \mathcal{O}\left(\left\lceil\frac{V_{0}}{\epsilon}\right\rceil + \left\lceil\frac{1}{\epsilon^{3/2}}\right\rceil\right).\label{sec_restarting:iter_Two_SA}
\end{align}
We now present the restarting scheme to improve the convergence of the bias (the first term) in \eqref{sec_restarting:iter_Two_SA}. Suppose that given a point $Z_{0} = [X_{0}^T,\;Y_{0}^T]^T$ we can guess a bound $\Delta_{0}$ such that $V_{0} \leq \Delta_{0}$.  Then given an infinite sequence of samples $\{\Abf_{ij}(\xi_{k})\}$ and $\{b_{i}(\xi_{k})\}$ for $i,j = 1,2$, the restarting method is formally stated in Algorithm \ref{alg:restarting}. The complexity of this algorithm is presented in the following lemma, where we show that the bias (the term associated with the initial condition) converges to zero exponentially. The analysis of this lemma is adopted from the work in \cite{Lan2016_GSC}.  
\begin{algorithm}\label{alg:restarting}
\caption{Restarting Linear Two-Time-Scale {\sf SA}}
\begin{enumerate}[leftmargin = 4.5mm]
\item Let a point $\tilde{Z}_{0} = [X_{0}^T,\;Y_{0}^T]^T$ and a bound $\Delta_{0}$ such that $V_{0} \leq \Delta_{0}$ 
\item For $k = 1,2,\ldots$
\begin{enumerate}[leftmargin = 6mm]
\item Run $N_{k}$ iterations of \eqref{alg:XY} with $Z_{0} =\tilde{Z}_{k-1}$ and the step sizes in \eqref{thm_rate:stepsizes}, where 
\begin{align}
N_{k} = \left\lceil\max\left\{4\Psi_{1},  \frac{\Psi_{2}^{3/2}}{\Delta_{0}^{3/2}2^{-3(k+1)/2}}\right\} \right\rceil. \label{sec_restarting:Ns_bound}
\end{align}
\item Set $\tilde{Z}_{k} =  Z_{N_k}$.
\end{enumerate} 
\end{enumerate}
\end{algorithm}

\begin{lem}\label{lem:restarting}
Suppose that Assumptions \ref{assump:stationary}--\ref{assump:mix} hold. Let $\{\tilde{Z}_{k} = [\tilde{X}_{k}^T,\tilde{Y}_{k}^T]^T\}$ be generated by the restarting linear two-time scale algorithm. Then we have  
\begin{align}
\tilde{V}_{k} = \Eset\left[\|\tilde{Y}_{k}\|^2\right] + \frac{1}{2\gamma\rho}\frac{\beta_{k}}{\alpha_{k}}\Eset\left[\|\tilde{X}_{k}\|^2\right] \leq \Delta_{k}  \triangleq \Delta_{0}2^{-k}. \label{lem_restarting:Ineq1}
\end{align}
In addition, this restarting method will find a solution  such that $\tilde{V}_{K}\leq \epsilon$ for any $\epsilon\in (0,\Delta_{0})$ in at most $K = \lceil\log(\Delta_{0}/\epsilon) \rceil$ epochs. That means, the total number of iterations performed by this restarting linear two-time-scale {\sf SA} to find such an $\epsilon$-solution is bounded by 
\begin{align}
K(\epsilon) = \Ocal\left( \left\lceil \log\left(\frac{\Delta_{0}}{\epsilon}\right) \right\rceil + \left\lceil \frac{1}{\epsilon^{3/2}}\right\rceil\right).\label{lem_restarting:Ineq2}
\end{align} 
\end{lem}

\begin{proof}
We show Eq.\  \eqref{lem_restarting:Ineq1} by induction.  For $k = 0$, clearly we have $\tilde{V}_{0}\leq \Delta_{0}$ . Assume that for some $s\geq 1$ we have $\tilde{V}_{k-1}\leq \Delta_{k-1} = \Delta_{0}2^{-(k-1)}$. We first have
\begin{align*}
\frac{\Psi_{1}\tilde{V}_{k-1}}{N_{k}}  \leq \frac{2\Psi_{1}\Delta_{k}}{N_{k}} \leq \frac{\Delta_{k}}{2}.
\end{align*}
Second,
\begin{align*}
\frac{\Psi_{2}}{N_{k}^{2/3}}  \leq \frac{\Psi_{2}\Delta_{0}2^{-k-1}}{\Psi_{2}} =  \frac{\Delta_{k}}{2}\cdot 
\end{align*}
By Eq.\ \eqref{sec_restarting:Ineq1} we have 
\begin{align*}
\tilde{V}_{N_{k}} &\leq \frac{\Psi_{1}\tilde{V}_{k-1}}{N_{k}} + \frac{\Psi_{2}}{N_{k}}\leq \Delta_{k}, 
\end{align*} 
which concludes our induction proof. Thus, to find an $\epsilon\in(0,\Delta_{0})$ solution we need $K = \lceil \log(\Delta_{0}/\epsilon) \rceil$ and the total number of iterations is given by 
\begin{align*}
K(\epsilon) &= \sum_{k=1}^{K}N_{k} = \sum_{k=1}^{K}\left[4\Psi_{1} + \frac{\Psi_{2}^{3/2}}{\Delta_{0}^{3/2}2^{-3(k+1)/2}}\right]\notag\\ 
&= 4\Psi_{1}\left\lceil \log\left(\frac{\Delta_{0}}{\epsilon}\right) \right\rceil  +   \left(\frac{2\Psi_{2}}{\Delta_{0}}\right)^{3/2}\sum_{k=1}^{K}2^{3k/2}\\
&\leq 4\Psi_{1}\left\lceil \log\left(\frac{\Delta_{0}}{\epsilon}\right) \right\rceil  +   \left(\frac{2\Psi_{2}}{\Delta_{0}}\right)^{3/2}\left(2^{3/2} + \int_{1}^{K}2^{3k/2}dk\right) \notag\\
&\leq 4\Psi_{1}\left\lceil \log\left(\frac{\Delta_{0}}{\epsilon}\right) \right\rceil  +   2\left(\frac{2\Psi_{2}}{\Delta_{0}}\right)^{3/2}2^{3K/2}\notag\\ 
&\leq 4\Psi_{1}\left\lceil \log\left(\frac{\Delta_{0}}{\epsilon}\right) \right\rceil + \left(4\Psi_{2}\right)^{3/2}\left\lceil \frac{1}{\epsilon^{3/2}}\right\rceil.
\end{align*}
\end{proof}

\begin{remark}
As can be seen from Eqs. \eqref{sec_restarting:iter_Two_SA} and \eqref{lem_restarting:Ineq2} that the restarting method requires much smaller number of iterations to eliminate the impact of the bias than the one without restarting. Restarting scheme also keeps the step sizes from getting  too small, which is benefit for the practical implementation of the algorithm. In addition, using restarting method with time-varying step sizes we get the same complexity as compared to the one using constant step sizes studied in \cite{GuptaSY2019_twoscale}.  

We note that the main practical challenge in implementing the restarting method is to determine the integer $N_{k}$ given in \eqref{sec_restarting:Ns_bound}. This number in general depends on the unknown mixing time $\tau(\alpha_{k})$ of the underlying Markov chain. One way to circumvent this issue is to utilize the technique recently studied in \cite{GuptaSY2019_twoscale} for the linear two-time-scale methods under constant step sizes. In particular, to decide the restarting time $N_{k}$ one can consider the time where the bias and variance in Eq.\ \eqref{sec_restarting:Ineq1} are equal, i.e., $\Psi_{1}V_{0} = \Psi_{2}k^{1/3}$. The technique in \cite{GuptaSY2019_twoscale} helps to compute this quantity by deciding the time when the curve of the bias term is horizontal (when the Markov chain becomes close to its steady-state). We refer interested readers to \cite{GuptaSY2019_twoscale} for more details of this interesting technique. Then instead of reducing the constant step sizes as suggested in \cite{GuptaSY2019_twoscale}, we restart our time-varying step sizes as shown in Algorithm \ref{alg:restarting}. Decreasing the step sizes might make the progress of the algorithm become slow since the step sizes are small as the number of iteration getting bigger.               
\end{remark}



\section{Proofs of Technical Lemmas \ref{lem:Xhat2}--\ref{lem:Zhat2_bounded}}\label{sec:lem_proofs}
\subsection{Preliminaries}
We first provide the recursive updates of $\Xhat_{k}$ and $\Yhat_{k}$ based on Eq.\  \eqref{alg:XY_reform}. This lemma was studied in \cite{DoanM2019}, so its proof is omitted here for brevity.  
\begin{lem}\label{lem:xyhat}
The sequences $\{\Xhat_{k},\Yhat_{k}\}$ defined in \eqref{alg:XY_hat} satisfy
\begin{align}
\Xhat_{k+1} &= \left(\Ibf - \alpha_{k}\Abf_{11} - \beta_{k}\Abf_{11}^{-1}\Abf_{12}\Abf_{21}\right)\Xhat_{k} + \alpha_{k}\epsilon_{k} + \beta_{k}\Abf_{11}^{-1}\Abf_{12}\psi_{k} - \beta_{k}\Abf_{11}^{-1}\Abf_{12}\Delta\Yhat_{k}\label{lem_xyhat:x}\\
\Yhat_{k+1} &= \left(\Ibf - \beta_{k}\Delta\right)\Yhat_{k} - \beta_{k}\Abf_{21}\Xhat_{k} + \beta_{k}\psi_{k},\label{lem_xyhat:y}
\end{align}
where $\Delta = \Abf_{22}-\Abf_{21}\Abf_{11}^{-1}\Abf_{12}$ satisfies Assumption \ref{assump:Hurwitz}.
\end{lem}
Next, we provide upper bounds for the size of the noise $\epsilon_{k}$ and $\psi_{k}$ where recall that
\begin{align*}
\Zhat_{k} = \left[\begin{array}{c}
     \Xhat_{k} \\
     \Yhat_{k} 
\end{array}\right].
\end{align*} 

\begin{lem}\label{lem:noise_bound}
Suppose that Assumption \ref{assump:bounded} holds. Let $\{\alpha_{k},\beta_{k}\}$ be two sequences of nonnegative and nonincreasing step sizes and $\Kcal_{1}^*$ satisfy \eqref{notation:K1*}. Then for all $k\geq \Kcal_{1}^*$ 
\begin{align}
\left\|\left[\begin{array}{c}
 \epsilon_{k} \\
\psi_{k} 
\end{array}\right]\right\|\leq \frac{(8\lambda_{1}+1)}{2\lambda_{1}}\|\Zhat_{k}\| + \frac{(8\lambda_{1}+1)(2B+\|Y^*\|)}{2\lambda_{1}^2}\cdot    \label{lem_noise_bound:Ineq}
\end{align}
\end{lem}

\begin{proof}
Let $Z_{k} = [X_{k}^T,Y_{k}^T]^T$ and by using Eq.\ \eqref{analysis:noise} we have
\begin{align*}
\left[\begin{array}{c}
 \epsilon_{k} \\
\psi_{k} 
\end{array}\right] = \left[\begin{array}{cc}
\Abf_{11}(\xi_{k})-\Abf_{11} & \Abf_{12}(\xi_{k}) - \Abf_{12} \\
\Abf_{21}(\xi_{k}) - \Abf_{21}     & \Abf_{22}(\xi_{k}) - \Abf_{22} 
\end{array}\right] Z_{k} - \left[\begin{array}{c}
b_{1}(\xi_{k}) - b_{1}     \\
b_{2}(\xi_{k}) - b_{2}
\end{array}\right],
\end{align*}
which implies Eq.\ \eqref{lem_noise_bound:Ineq}, i.e., 
\begin{align*}
\left\|\left[\begin{array}{c}
 \epsilon_{k} \\
\psi_{k} 
\end{array}\right]\right\|&\leq \left\|\left[\begin{array}{cc}
\Abf_{11}(\xi_{k})-\Abf_{11} & \Abf_{12}(\xi_{k}) - \Abf_{12} \\
\Abf_{21}(\xi_{k}) - \Abf_{21}     & \Abf_{22}(\xi_{k}) - \Abf_{22} 
\end{array}\right]\right\|\|Z_{k}\| + \left\| \left[\begin{array}{c}
b_{1}(\xi_{k}) - b_{1}     \\
b_{2}(\xi_{k}) - b_{2}
\end{array}\right]\right\| \notag\\
& \leq 2\|Z_{k}\| + 4B \leq \frac{(8\lambda_{1}+1)}{2\lambda_{1}}\|\Zhat_{k}\| + \frac{(8\lambda_{1}+1)(B+\|Y^*\|)}{2\lambda_{1}^2} + 4B \notag\\
&\leq \frac{(8\lambda_{1}+1)}{2\lambda_{1}}\|\Zhat_{k}\| + \frac{(8\lambda_{1}+1)(2B+\|Y^*\|)}{2\lambda_{1}^2},
\end{align*}
where in the third inequality we use Eq.\ \eqref{lem_XYhat_bound:Eq1a} in Appendix and recall that $\lambda_{1}$ is the smallest singular value of $\Abf_{11}$. 
\end{proof}
Finally, to show the results in Lemmas \ref{lem:Xhat2}--\ref{lem:Zhat2_bounded}, we consider the following important results about the bias due to the  Markovian noise $\epsilon_{k}$ and $\psi_{k}$. Indeed, since we have $\Eset[\epsilon_{k}\,|\Fcal_{k}] \neq 0$ and $\Eset[\psi_{k}\,|\Fcal_{k}] \neq 0$ for $\Fcal_{k}$ containing all the history up to time $k$, we have to carefully quantify the sizes of the noise in our update. The following lemma is to achieve this goal. The analysis of these results is quite complicated, so we delay it to Appendix \ref{apx:proof_lemma_XY_nois} for an ease of exposition. 
\begin{lem}\label{lem:XY_noise}
Suppose that Assumptions \ref{assump:stationary}--\ref{assump:mix} holds. Let $\{\alpha_{k},\beta_{k}\}$ be two sequences of nonnegative and nonincreasing step sizes and $\Kcal_{1}^*$ satisfy \eqref{notation:K1*}. Then for all $k\geq \Kcal_{1}^*$
\begin{align}
\Eset[\epsilon_{k}^T\Xhat_{k}]&\leq \frac{3(8\lambda_{1}+1)^{5}\tau(\alpha_{k})}{2\lambda_{1}^{5}}\alpha_{k-\tau(\alpha_{k})}\Eset\left[\|\Zhat_{k}\|^2\right] \notag\\ 
&\qquad + \frac{3(8\lambda_{1}+1)^5(2B+\|Y^*\|)^2\tau(\alpha_{k})}{
\lambda_{1}^7}\alpha_{k-\tau(\alpha_{k})}.   \label{lem_XY_noise:Ineq1a}\\
\Eset[\psi_{k}^T\Abf_{11}^{-1}\Abf_{12}\Xhat_{k}]&\leq  \frac{3(8\lambda_{1}+1)^{6}\tau(\alpha_{k})}{8\lambda_{1}^{5}}\alpha_{k-\tau(\alpha_{k})}\Eset\left[\|\Zhat_{k}\|^2\right] \notag\\ 
&\qquad + \frac{3(8\lambda_{1}+1)^5(2B+\|Y^*\|)^2\tau(\alpha_{k})}{4\lambda_{1}^8}\alpha_{k-\tau(\alpha_{k})}. \label{lem_XY_noise:Ineq1b}\\
\Eset[\psi_{k}^T\Yhat_{k}]&\leq \frac{3(8\lambda_{1}+1)^{5}\tau(\alpha_{k})}{2\lambda_{1}^{5}}\alpha_{k-\tau(\alpha_{k})}\Eset\left[\|\Zhat_{k}\|^2\right]\notag\\ 
& \qquad + \frac{3(8\lambda_{1}+1)^5(2B+\|Y^*\|)^2\tau(\alpha_{k})}{
\lambda_{1}^7}\alpha_{k-\tau(\alpha_{k})}.\label{lem_XY_noise:Ineq1c}
\end{align}
\end{lem}

Using Lemmas \ref{lem:noise_bound} and \ref{lem:XY_noise}, we now proceed with our main analysis in this section.
\subsection{Proof of Lemma \ref{lem:Xhat2} }

\begin{proof}
For convenience, let $h_k$ be defined as
\begin{align*}
h_{k} = (\Ibf - \alpha_{k}\Abf_{11})\Xhat_{k} - \beta_{k}\Abf_{11}^{-1}\Abf_{12}\big(\Abf_{21}\Xhat_{k} + \Delta\Yhat_{k}\big),
\end{align*}
where $\Delta = \Abf_{22}-\Abf_{21}\Abf_{11}^{-1}\Abf_{12}$ satisfying Assumption \ref{assump:Hurwitz}. Thus, by Eq.\ \eqref{lem_xyhat:x} we have
\begin{align*}
\Xhat_{k+1} = h_{k} + \alpha_{k}\epsilon_{k}+\beta_{k}\Abf_{11}^{-1}\Abf_{12}\psi_{k},    
\end{align*}
which gives
\begin{align}
\Eset\left[\|\Xhat_{k+1}\|^2\right] &= \Eset\left[\|h_{k}\|^2\right] + \Eset\left[\|\alpha_{k}\epsilon_{k}+\beta_{k}\Abf_{11}^{-1}\Abf_{12}\psi_{k}\|^2\right]\notag\\ 
&\qquad + 2\Eset\left[h_{k}^T(\alpha_{k}\epsilon_{k}+\beta_{k}\Abf_{11}^{-1}\Abf_{12}\psi_{k})\right].    \label{lem_Xhat2:Eq1}
\end{align}
Recall that $\gamma > 0$ is the smallest eigenvalue of $\Abf_{11}$, $\lambda_{1}$ is the smallest singular value of $\Abf_{11}$, and $\sigma_{n}$ is the largest singular value of $\Delta$. Using Assumption \ref{assump:bounded}, i.e., $\|\Abf_{ij}\|\leq 1/4$ for all $i,j=1,2$, we first consider
\begin{align*}
\|(\Ibf - \alpha_{k}\Abf_{11})\Xhat_{k}\|^2 &= \|\Xhat_{k}\|^2 - \alpha_{k}\Xhat_{k}^T(\Abf_{11}^T+\Abf_{11})\Xhat_{k} + \alpha_{k}^2\|\Abf_{11}\Xhat_{k}\|^2\notag\\    
&\leq (1-2\gamma\alpha_{k})\|\Xhat_{k}\|^2 + \frac{\alpha_{k}^2}{16}\|\Xhat_{k}\|^2.
\end{align*}
Second, we have
\begin{align*}
\|\beta_{k}^2\Abf_{11}^{-1}(\Abf_{21}\Xhat_{k}+\Delta\Yhat_{k})\|^2&\leq\frac{(1+\sigma_{n})^2}{\lambda_{1}^2}\beta_{k}^2\|\Zhat_{k}\|^2.    
\end{align*}
Third, using the Cauchy-Schwarz inequality we obtain
\begin{align*}
& - 2\beta_{k}\Xhat_{k}^T(\Ibf-\alpha_{k}\Abf_{11})^T\Abf_{11}^{-1}(\Abf_{21}\Xhat_{k}+\Delta\Yhat_{k}) \leq \gamma \alpha_{k} \|\Xhat_{k}\|^2 +\frac{(1+\alpha_{0})^2(1+\sigma_{n})^2\beta_{k}^2}{\gamma\lambda_{1}^2\alpha_{k}}\|\Zhat_{k}\|^2.  
\end{align*}
Using the previous three relations, we consider
\begin{align}
&\|h_{k}\|^2 = \|(\Ibf - \alpha_{k}\Abf_{11})\Xhat_{k} - \beta_{k}\Abf_{11}^{-1}\Abf_{12}\big(\Abf_{21}\Xhat_{k} + \Delta\Yhat_{k}\big)\|^2\notag\\
&= \|(\Ibf-\alpha_{k}\Abf_{11})\Xhat_{k}\|^2 + \beta_{k}^2\|\Abf_{11}^{-1}(\Abf_{21}\Xhat_{k}+\Delta\Yhat_{k})\|^2 - 2\beta_{k}\Xhat_{k}^T(\Ibf-\alpha_{k}\Abf_{11})^T\Abf_{11}^{-1}(\Abf_{21}\Xhat_{k}+\Delta\Yhat_{k})\notag\\
&\leq (1-2\gamma\alpha_{k})\|\Xhat_{k}\|^2 + \frac{\alpha_{k}^2}{16}\|\Xhat_{k}\|^2 + \frac{(1+\sigma_{n})^2}{\lambda_{1}^2}\beta_{k}^2\|\Zhat_{k}\|^2 +\gamma\alpha_{k}\|\Xhat_{k}\|^2 + \frac{(1+\alpha_{0})^2(1+\sigma_{n})^2\beta_{k}^2}{\gamma\lambda_{1}^2\alpha_{k}}\|\Zhat_{k}\|^2\notag\\
&\leq (1-\gamma\alpha_{k})\|\Xhat_{k}\|^2 + \frac{\alpha_{k}^2}{16}\|\Xhat_{k}\|^2 + \frac{(1+\sigma_{n})^2}{\lambda_{1}^2}\beta_{k}^2\|\Zhat_{k}\|^2   + \frac{(1+\alpha_{0})^2(1+\sigma_{n})^2\beta_{k}^2}{\gamma\lambda_{1}^2\alpha_{k}}\|\Zhat_{k}\|^2 .\label{lem_Xhat2:Eq1a}  
\end{align}
Next, using Eq.\ \eqref{lem_noise_bound:Ineq} and Assumption \ref{assump:bounded} we consider 
\begin{align}
\|\alpha_{k}\epsilon_{k} + \beta_{k}\Abf_{11}^{-1}\Abf_{12}\psi_{k}\|^2 &\leq 2\alpha_{k}^2\epsilon_{k}^2 + \frac{\beta_{k}^2}{2\lambda_{1}}\psi_{k}^2\leq \frac{(4\lambda_{1}+1)\alpha_{k}^2}{2\lambda_{1}}\left\|\left[\begin{array}{c}
 \epsilon_{k} \\
\psi_{k} 
\end{array}\right]\right\|^2\notag\\ 
&\leq \frac{(8\lambda_{1}+1)^3}{4\lambda_{1}^{3}}\alpha_{k}^2\|\Zhat_{k}\|^2 + \frac{(8\lambda_{1}+1)^{3}(2B+\|Y^*\|)^2}{4\lambda_{1}^{5}}\alpha_{k}^2\cdot
\label{lem_Xhat2:Eq1b}
\end{align}
Finally, we consider the last term on the right-hand side of Eq.\ \eqref{lem_Xhat2:Eq1}
\begin{align}
&2\Eset[h_{k}^{T}(\alpha_{k}\epsilon_{k}+\beta_{k}\Abf_{11}^{-1}\Abf_{12}\psi_{k})]\notag\\ 
&= 2\Eset[\Xhat_{k}^T(\alpha_{k}\epsilon_{k}+\beta_{k}\Abf_{11}^{-1}\Abf_{12}\psi_{k})] - 2\alpha_{k}\Eset[\Xhat_{k}^{T}\Abf_{11}^T(\alpha_{k}\epsilon_{k}+\beta_{k}\Abf_{11}^{-1}\Abf_{12}\psi_{k})]\notag\\ 
&\quad - 2\beta_{k}\Eset[(\Abf_{21}\Xhat_{k}+\Delta\Yhat_{k})^T(\Abf_{11}^{-1}\Abf_{12})^T(\alpha_{k}\epsilon_{k}+\beta_{k}\Abf_{11}^{-11}\Abf_{12}\psi_{k})]. \label{lem_Xhat2:Eq1c}
\end{align}
Using Eqs.\ \eqref{lem_XY_noise:Ineq1a} and \eqref{lem_XY_noise:Ineq1b}, consider the first term on the right-hand side of \eqref{lem_Xhat2:Eq1c} 
\begin{align}
&2\Eset\left[\Xhat_{k}^T(\alpha_{k}\epsilon_{k}+\beta_{k}\Abf_{11}^{-1}\Abf_{12}\psi_{k})\right]\notag\\
&\leq \frac{3(8\lambda_{1}+1)^{5}\tau(\alpha_{k})}{\lambda_{1}^{5}}\alpha_{k-\tau(\alpha_{k})}\alpha_{k}\Eset\left[\|\Zhat_{k}\|^2\right]+ \frac{3(8\lambda_{1}+1)^{6}\tau(\alpha_{k})}{4\lambda_{1}^{5}}\alpha_{k-\tau(\alpha_{k})}\beta_{k}\Eset\left[\|\Zhat_{k}\|^2\right]\notag\\ 
&\qquad + \frac{3(8\lambda_{1}+1)^5(2B+\|\Ybf^*\|)^2\tau(\alpha_{k})}{
\lambda_{1}^7}\alpha_{k-\tau(\alpha_{k})}\alpha_{k}\notag\\
&\qquad + \frac{3(8\lambda_{1}+1)^5(2B+\|\Ybf^*\|)^2\tau(\alpha_{k})}{2\lambda_{1}^8}\alpha_{k-\tau(\alpha_{k})}\beta_{k}\notag\\
& \leq \frac{6(8\lambda_{1}+1)^{5}\tau(\alpha_{k})}{\lambda_{1}^{6}}\alpha_{k-\tau(\alpha_{k})}\alpha_{k}\Eset\left[\|\Zhat_{k}\|^2\right] + \frac{6(8\lambda_{1}+1)^5(2B+\|\Ybf^*\|)^2\tau(\alpha_{k})}{
\lambda_{1}^8}\alpha_{k-\tau(\alpha_{k})}\alpha_{k}.\label{lem_Xhat2:Eq1c1}
\end{align}
Next, consider the second term on the right-hand side of Eq.\ \eqref{lem_Xhat2:Eq1c} by using Eq.\ \eqref{lem_noise_bound:Ineq} and Assumption \ref{assump:bounded}
\begin{align*}
& - 2\alpha_{k}\Eset\left[\Xhat_{k}^{T}\Abf_{11}^T(\alpha_{k}\epsilon_{k}+\beta_{k}\Abf_{11}^{-1}\Abf_{12}\psi_{k})\right] = -2\alpha_{k} \Eset\left[\Xhat_{k}^T\left[
\alpha_{k}\Abf_{11}^T\quad \beta_{k}\Abf_{11}^T\Abf_{11}^{-1}\Abf_{12}\right]\left[\begin{array}{c}
\epsilon_{k} \\
\psi_{k} 
\end{array}\right]\right]\notag\\  
&\quad\leq2\alpha_{k}\left(\frac{\alpha_{k}}{4} + \frac{\beta_{k}}{16\lambda_{1}}\right) \Eset\left[\|\Xhat_{k}\|\left\|\left[\begin{array}{c}
\epsilon_{k} \\
\psi_{k} 
\end{array}\right]\right\|\right]\notag\\
&\quad\stackrel{\eqref{lem_noise_bound:Ineq}}{\leq}  \frac{\alpha_{k}}{2\lambda_{1}}(\alpha_{k} + \beta_{k})\Eset\left[\|\Zhat_{k}\|\left(\frac{(8\lambda_{1}+1)}{2\lambda_{1}}\|\Zhat_{k}\| + \frac{(8\lambda_{1}+1)(2B+\|Y^*\|)}{2\lambda_{1}^2}\right)\right]\notag\\
&\quad= \frac{\alpha_{k}}{2\lambda_{1}}(\alpha_{k} + \beta_{k})\Eset\left[\frac{(8\lambda_{1}+1)}{2\lambda_{1}}\|\Zhat_{k}\|^2 + \frac{(8\lambda_{1}+1)(2B+\|Y^*\|)}{2\lambda_{1}^2}\|\Zhat_{k}\|\right],
\end{align*}
which by applying the inequality $2xy\leq x^2+y^2$ for $x,y\in\Rset$ to the last term on the right-hand side yields 
\begin{align}
& - 2\alpha_{k}\Eset\left[\Xhat_{k}^{T}\Abf_{11}^T(\alpha_{k}\epsilon_{k}+\beta_{k}\Abf_{11}^{-1}\Abf_{12}\psi_{k})\right] \notag\\
&\quad\leq \frac{\alpha_{k}}{2\lambda_{1}}(\alpha_{k} + \beta_{k})\Eset\left[\frac{(8\lambda_{1}+1)}{2\lambda_{1}}\|\Zhat_{k}\|^2 +\frac{(8\lambda_{1}+1)}{4\lambda_{1}}\|\Zhat_{k}\|^2 + \frac{(8\lambda_{1}+1)(2B+\|Y^*\|)^2}{4\lambda_{1}^2}\right]
\notag\\ 
&\quad\leq \frac{(8\lambda_{1}+1)}{2\lambda_{1}^2}(\alpha_{k}^2 + \alpha_{k}\beta_{k})\Eset\left[\|\Zhat_{k}\|^2\right] + \frac{(8\lambda_{1}+1)(2B+\|Y^*\|)^2}{8\lambda_{1}^3}(\alpha_{k}^2 + \alpha_{k}\beta_{k}).\label{lem_Xhat2:Eq1c2}
\end{align}
Similarly, consider the last term on the right-hand side of Eq.\ \eqref{lem_Xhat2:Eq1c}
\begin{align*}
&- 2\beta_{k}(\Abf_{21}\Xhat_{k}+\Delta\Yhat_{k})^T(\Abf_{11}^{-1}\Abf_{12})^T(\alpha_{k}\epsilon_{k}+\beta_{k}\Abf_{11}^{-11}\Abf_{12}\psi_{k})\notag\\
&= -2 \beta_{k}\left[\begin{array}{c}
     \Xhat_{k}  \\
     \Yhat_{k} 
\end{array}\right]^T\left[\begin{array}{c}
\Abf_{21}^T\\
\Delta^T
\end{array}\right](\Abf_{11}^{-1}\Abf_{12})^T\left[\alpha_{k}\Ibf\quad \beta_{k}\Abf_{11}^{-1}\Abf_{12}\right]\left[\begin{array}{c}
    \epsilon_{k}  \\
     \psi_{k}
\end{array}\right] \notag\\
&\leq 2\beta_{k}\left\|\left[\begin{array}{c}
\Abf_{21}^T\\
\Delta^T
\end{array}\right]\right\|\left\|\Abf_{11}^{-1}\Abf_{12}\right\|\left\|\left[\alpha_{k}\Ibf\quad \beta_{k}\Abf_{11}^{-1}\Abf_{12}\right]\right\|\left\|\left[\begin{array}{c}
     \Xhat_{k}  \\
     \Yhat_{k} 
\end{array}\right]\right\|\left\|\left[\begin{array}{c}
     \epsilon_{k}  \\
     \psi_{k}
\end{array}\right]\right\|\notag\\
&\leq 2\beta_{k}\frac{1}{4\lambda_{1}}\left(\frac{1}{4}+\sigma_{n}\right)\left(\alpha_{k}+\frac{\beta_{k}}{4\lambda_{1}}\right)\|\Zhat_{k}\|\left\|\left[\begin{array}{c}
     \epsilon_{k}  \\
     \psi_{k}
\end{array}\right]\right\| \leq \frac{(1+4\sigma_{n})}{32\lambda_{1}^2}\beta_{k}(\alpha_{k}+\beta_{k})\|\Zhat_{k}\|\left\|\left[\begin{array}{c}
     \epsilon_{k}  \\
     \psi_{k}
\end{array}\right]\right\|,
\end{align*}
which similar to Eq.\ \eqref{lem_Xhat2:Eq1c2} (by using Eq.\ \eqref{lem_noise_bound:Ineq} again)  yields 
\begin{align}
&- 2\beta_{k}\Eset[(\Abf_{21}\Xhat_{k}+\Delta\Yhat_{k})^T(\Abf_{11}^{-1}\Abf_{12})^T(\alpha_{k}\epsilon_{k}+\beta_{k}\Abf_{11}^{-11}\Abf_{12}\psi_{k})]\notag\\ 
&\quad \leq \frac{(1+4\sigma_{n})(8\lambda_{1}+1)}{32\lambda_{1}^3}(\beta_{k}^2 + \alpha_{k}\beta_{k})\Eset\left[\|\Zhat_{k}\|^2\right] + \frac{(1+4\sigma_{n})(8\lambda_{1}+1)(2B+\|Y^*\|)^2}{128\lambda_{1}^4}(\beta_{k}^2 + \alpha_{k}\beta_{k}).
\label{lem_Xhat2:Eq1c3}
\end{align}
Using Eqs.\ \eqref{lem_Xhat2:Eq1c1} --\eqref{lem_Xhat2:Eq1c3} into Eq.\ \eqref{lem_Xhat2:Eq1c} yields
\begin{align}
&2\Eset\left[h_{k}^{T}(\alpha_{k}\epsilon_{k}+\beta_{k}\Abf_{11}^{-1}\Abf_{12}\psi_{k})\right]\notag\\ 
& \leq \frac{6(8\lambda_{1}+1)^{5}\tau(\alpha_{k})}{\lambda_{1}^{6}}\alpha_{k-\tau(\alpha_{k})}\alpha_{k}\Eset\left[\|\Zhat_{k}\|^2\right]+ \frac{6(8\lambda_{1}+1)^5(2B+\|\Ybf^*\|)^2\tau(\alpha_{k})}{
\lambda_{1}^8}\alpha_{k-\tau(\alpha_{k})}\alpha_{k}  \notag\\
&\quad +\frac{(8\lambda_{1}+1)}{2\lambda_{1}^2}(\alpha_{k}^2 + \alpha_{k}\beta_{k})\Eset\left[\|\Zhat_{k}\|^2\right]+ \frac{(8\lambda_{1}+1)(2B+\|Y^*\|)^2}{8\lambda_{1}^3}(\alpha_{k}^2 + \alpha_{k}\beta_{k})\notag\\
&\quad +\frac{(1+4\sigma_{n})(8\lambda_{1}+1)}{32\lambda_{1}^3}(\beta_{k}^2 + \alpha_{k}\beta_{k})\Eset\left[\|\Zhat_{k}\|^2\right] + \frac{(1+4\sigma_{n})(8\lambda_{1}+1)(2B+\|Y^*\|)^2}{128\lambda_{1}^4}(\beta_{k}^2 + \alpha_{k}\beta_{k})\notag\\
&\leq \frac{6(8\lambda_{1}+1)^{5}\tau(\alpha_{k})}{\lambda_{1}^{6}}\alpha_{k-\tau(\alpha_{k})}\alpha_{k}\Eset\left[\|\Zhat_{k}\|^2\right]+ \frac{6(8\lambda_{1}+1)^5(2B+\|\Ybf^*\|)^2\tau(\alpha_{k})}{
\lambda_{1}^8}\alpha_{k-\tau(\alpha_{k})}\alpha_{k}\notag\\
&\quad +\frac{(1+4\sigma_{n})(8\lambda_{1}+1)}{2\lambda_{1}^3}(\beta_{k} + \alpha_{k})^2\Eset\left[\|\Zhat_{k}\|^2\right]  + \frac{3(1+4\sigma_{n})(8\lambda_{1}+1)(2B+\|Y^*\|)^2}{8\lambda_{1}^4}(\beta_{k} + \alpha_{k})^2.
\label{lem_Xhat2:Eq1d}
\end{align}
Thus, we now using Eqs.\ \eqref{lem_Xhat2:Eq1a}, \eqref{lem_Xhat2:Eq1b}, and \eqref{lem_Xhat2:Eq1d} into Eq.\ \eqref{lem_Xhat2:Eq1} to have Eq.\ \eqref{lem_Xhat2:Ineq}, i.e.,
\begin{align*}
\Eset\left[\|\Xhat_{k+1}\|^2\right] &\leq (1-\gamma\alpha_{k})\Eset\left[\|\Xhat_{k}\|^2\right] + \frac{\alpha_{k}^2}{16}\Eset\left[\|\Xhat_{k}\|^2\right] + \frac{(1+\sigma_{n})^2}{\lambda_{1}^2}\beta_{k}^2\Eset\left[\|\Zhat_{k}\|^2\right]\notag\\
&\quad  + \frac{(1+\alpha_{0})^2(1+\sigma_{n})^2\beta_{k}^2}{\gamma\lambda_{1}^2\alpha_{k}}\Eset\left[\|\Zhat_{k}\|^2\right] + \frac{(8\lambda_{1}+1)^3}{4\lambda_{1}^{3}}\alpha_{k}^2\Eset\left[\|\Zhat_{k}\|^2\right]  \notag\\
&\quad + \frac{(8\lambda_{1}+1)^{3}(2B+\|Y^*\|)^2}{4\lambda_{1}^{5}}\alpha_{k}^2 + \frac{6(8\lambda_{1}+1)^{5}\tau(\alpha_{k})}{\lambda_{1}^{6}}\alpha_{k-\tau(\alpha_{k})}\alpha_{k}\Eset\left[\|\Zhat_{k}\|^2\right] \notag\\
&\quad + \frac{6(8\lambda_{1}+1)^5(2B+\|\Ybf^*\|)^2\tau(\alpha_{k})}{
\lambda_{1}^8}\alpha_{k-\tau(\alpha_{k})}\alpha_{k} +\frac{(1+4\sigma_{n})(8\lambda_{1}+1)}{2\lambda_{1}^3}(\beta_{k} + \alpha_{k})^2\Eset\left[\|\Zhat_{k}\|^2\right] \notag\\
&\quad + \frac{3(1+4\sigma_{n})(8\lambda_{1}+1)(2B+\|Y^*\|)^2}{8\lambda_{1}^4}(\beta_{k} + \alpha_{k})^2\notag\\
&\leq (1-\gamma\alpha_{k})\Eset\left[\|\Xhat_{k}\|^2\right] + \frac{\alpha_{k}^2}{16}\Eset\left[\|\Xhat_{k}\|^2\right] + \frac{(1+\alpha_{0})^2(1+\sigma_{n})^2\beta_{k}^2}{\gamma\lambda_{1}^2\alpha_{k}}\Eset\left[\|\Zhat_{k}\|^2\right]\notag\\
&\quad + \frac{6(1+\sigma_{n})^2(8\lambda_{1}+1)^{5}}{\lambda_{1}^{6}}\left(\tau(\alpha_{k})\alpha_{k-\tau(\alpha_{k})}\alpha_{k} + \alpha_{k}^2+\alpha_{k}\beta_{k}\right)\Eset\left[\|\Zhat_{k}\|^2\right]\notag\\
&\quad + \frac{6(1+4\sigma_{n})(1+8\lambda_{1})^5(2B+\|Y^*\|)^2}{\lambda_{1}^8}\left(\tau(\alpha_{k})\alpha_{k-\tau(\alpha_{k})}\alpha_{k}+\alpha_{k}^2\right),
\end{align*}
where in the last inequality we use $\beta_{k}\leq \alpha_{k}$.
\end{proof}

\subsection{Proof of Lemma \ref{lem:Yhat2} }

\begin{proof}
Recall that $\sigma_{1}\leq\ldots\leq\sigma_{n}$ are the singular values of $\Delta$ and $\rho$ is the smallest eigenvalue of $\Delta$. By Eq.\ \eqref{lem_xyhat:y} we first consider
\begin{align}
&\Eset\left[\|\Yhat_{k+1}\|^2\right] = \Eset\left[\|\left(\Ibf - \beta_{k}\Delta\right)\Yhat_{k} - \beta_{k}\Abf_{21}\Xhat_{k} + \beta_{k}\psi_{k}\|^2\right]\notag\\
&= \Eset\left[\|\left(\Ibf - \beta_{k}\Delta\right)\Yhat_{k} - \beta_{k}\Abf_{21}\Xhat_{k}\|^2] + \beta_{k}^2\Eset[\|\psi_{k}\|^2\right] + 2\beta_{k}\Eset\left[\psi_{k}^T\left((\Ibf - \beta_{k}\Delta)\Yhat_{k} - \beta_{k}\Abf_{21}\Xhat_{k}\right)\right]. \label{lem_Yhat2:Eq1}
\end{align}
Next, using Assumption \ref{assump:bounded} we consider the following three relations 
\begin{align*}
&1)\quad \|(\Ibf-\beta_{k}\Delta)\Yhat_{k}\|^2 \leq (1-2\rho\beta_{k})\|\Yhat_{k}\|^2 + \sigma_{n}^2\beta_{k}^2\|\Yhat_{k}\|^2.\\ 
&2)\quad \|\beta_{k}\Abf_{21}\Xhat_{k}\|^2 \leq \frac{\beta_{k}^2}{16}\|\Xhat_{k}\|^2.\\
&3)\quad -2\beta_{k}\Yhat_{k}^T(\Ibf-\beta_{k}\Delta)^T\Abf_{21}\Xhat_{k}\leq \frac{\beta_{k}}{2}\|\Yhat_{k}\|\|\Xhat_{k}\| + \frac{\sigma_{n}\beta_{k}^2}{2}\|\Yhat_{k}\|\|\Xhat_{k}\|  \leq \rho\beta_{k}\|\Yhat_{k}\|^2 + \frac{\beta_{k}}{4\rho}\|\Xhat_{k}\|^2 +\sigma_{n}\beta_{k}^2\|\Zhat_{k}\|^2.
\end{align*}
Using the preceding three relations, we consider
\begin{align}
&\|\left(\Ibf - \beta_{k}\Delta\right)\Yhat_{k} - \beta_{k}\Abf_{21}\Xhat_{k}\|^2=   \|(\Ibf-\beta_{k}\Delta)\Yhat_{k}\|^2 +\|\beta_{k}\Abf_{21}\Xhat_{k}\|^2 -2\beta_{k}\Yhat_{k}^T(\Ibf-\beta_{k}\Delta)^T\Abf_{21}\Xhat_{k}\notag\\
&\leq  (1-2\rho\beta_{k})\|\Yhat_{k}\|^2 + \sigma_{n}^2\beta_{k}^2\|\Yhat_{k}\|^2 + \frac{\beta_{k}^2}{16}\|\Xhat_{k}\|^2 + \sigma\beta_{k}\|\Yhat_{k}\|^2 + \frac{\beta_{k}}{4\rho}\|\Xhat_{k}\|^2 +\sigma_{n}\beta_{k}^2\|\Zhat_{k}\|^2\notag\\
& = (1-\rho\beta_{k})\|\Yhat_{k}\|^2 + \frac{\beta_{k}}{4\rho}\|\Xhat_{k}\|^2 +(\sigma_{n}+1)^2\beta_{k}^2\|\Zhat_{k}\|^2
.\label{lem_Yhat2:Eq1a}
\end{align}
Second, using Eq. \eqref{lem_noise_bound:Ineq}  we obtain
\begin{align}
\|\psi_{k}\|^2 &\leq \frac{(8\lambda_{1}+1)^2}{2\lambda_{1}^2}\|\Zhat_{k}\|^2 + \frac{(8\lambda_{1}+1)^2(2B+\|Y^*\|)^2}{2\lambda_{1}^4}\cdot \label{lem_Yhat2:Eq1b}
\end{align}
Finally, the last term on the right-hand side of Eq.\ \eqref{lem_Yhat2:Eq1} can be bounded by using Eqs.\ \eqref{lem_XY_noise:Ineq1c} and \eqref{lem_noise_bound:Ineq} as
\begin{align}
&2\beta_{k}\Eset\left[\psi_{k}^T\left((\Ibf - \beta_{k}\Delta)\Yhat_{k} - \beta_{k}\Abf_{21}\Xhat_{k}\right)\right] =  2\beta_{k}\Eset[\psi_{k}^T\Yhat_{k}] - 2\beta_{k}^2\Eset\left[\psi_{k}^T[\Delta\quad \Abf_{21}]
\Zhat_{k}\right] \notag\\
&\stackrel{\eqref{lem_XY_noise:Ineq1c}}{\leq} \frac{3(8\lambda_{1}+1)^{5}\tau(\alpha_{k})}{\lambda_{1}^{5}}\alpha_{k-\tau(\alpha_{k})}\beta_{k}\Eset\left[\|\Zhat_{k}\|^2\right] \notag\\
&\qquad + \frac{6(8\lambda_{1}+1)^5(2B+\|\Ybf^*\|)^2\tau(\alpha_{k})}{
\lambda_{1}^7}\alpha_{k-\tau(\alpha_{k})}\beta_{k} + \frac{(4\sigma_{n}+1)}{2}\beta_{k}^2\Eset[\|\psi_{k}\|\|\Zhat_{k}\|]\notag\\
& \stackrel{\eqref{lem_noise_bound:Ineq}}{\leq} \frac{3(8\lambda_{1}+1)^{5}\tau(\alpha_{k})}{\lambda_{1}^{5}}\alpha_{k-\tau(\alpha_{k})}\beta_{k}\Eset\left[\|\Zhat_{k}\|^2\right] + \frac{6(8\lambda_{1}+1)^5(2B+\|\Ybf^*\|)^2\tau(\alpha_{k})}{
\lambda_{1}^7}\alpha_{k-\tau(\alpha_{k})}\beta_{k}\notag\\
&\qquad + \frac{(4\sigma_{n}+1)}{2}\beta_{k}^2\Eset\left[\frac{(8\lambda_{1}+1)}{2\lambda_{1}}\|\Zhat_{k}\|^2 + \frac{(8\lambda_{1}+1)(2B+\|Y^*\|)}{2\lambda_{1}^2}\|\Zhat_{k}\|\right]\notag\\
& \leq \frac{3(8\lambda_{1}+1)^{5}\tau(\alpha_{k})}{\lambda_{1}^{5}}\alpha_{k-\tau(\alpha_{k})}\beta_{k}\Eset\left[\|\Zhat_{k}\|^2\right] + \frac{6(8\lambda_{1}+1)^5(2B+\|\Ybf^*\|)^2\tau(\alpha_{k})}{
\lambda_{1}^7}\alpha_{k-\tau(\alpha_{k})}\beta_{k}\notag\\
&\qquad + \frac{(4\sigma_{n}+1)}{4}\beta_{k}^2\Eset\left[\frac{(8\lambda_{1}+1)}{\lambda_{1}}\|\Zhat_{k}\|^2 + \frac{(8\lambda_{1}+1)(2B+\|Y^*\|)^2}{4\lambda_{1}^3}\right]\notag\\
&\leq \frac{6(4\sigma_{n}+1)(8\lambda_{1}+1)^{5}}{\lambda_{1}^{5}}\Big(\tau(\alpha_{k})\alpha_{k-\tau(\alpha_{k})}\beta_{k} + \beta_{k}^2\Big)\Eset\left[\|\Zhat_{k}\|^2\right]\notag\\
&\qquad + \frac{12(4\sigma_{n}+1)(8\lambda_{1}+1)^5(2B+\|\Ybf^*\|)^2}{
\lambda_{1}^7}\Big(\tau(\alpha_{k})\alpha_{k-\tau(\alpha_{k})}\beta_{k}+\beta_{k}^2\Big). 
\label{lem_Yhat2:Eq1c}
\end{align}
Thus, using Eqs. \eqref{lem_Yhat2:Eq1a}--\eqref{lem_Yhat2:Eq1c} into Eq.\ \eqref{lem_Yhat2:Eq1} yields Eq.\ \eqref{lem_Yhat2:Ineq}, i.e.,
\begin{align*}
\Eset\left[\|\Yhat_{k+1}\|^2\right] &\leq (1-\rho\beta_{k})\|\Yhat_{k}\|^2 + \frac{\beta_{k}}{4\rho}\Eset\left[\|\Xhat_{k}\|^2\right] +(\sigma_{n}+1)^2\beta_{k}^2\Eset\left[\|\Zhat_{k}\|^2\right]+\frac{(8\lambda_{1}+1)^2}{2\lambda_{1}^2}\beta_{k}^2\Eset\left[\|\Zhat_{k}\|^2\right]  \notag\\
&\quad + \frac{(8\lambda_{1}+1)^2(2B+\|Y^*\|)^2}{2\lambda_{1}^4}\beta_{k}^2 + \frac{6(4\sigma_{n}+1)(8\lambda_{1}+1)^{5}}{\lambda_{1}^{5}}\Big(\tau(\alpha_{k})\alpha_{k-\tau(\alpha_{k})}\beta_{k} + \beta_{k}^2\Big)\Eset\left[\|\Zhat_{k}\|^2\right]\notag\\
&\quad + \frac{12(4\sigma_{n}+1)(8\lambda_{1}+1)^5(2B+\|\Ybf^*\|)^2}{
\lambda_{1}^7}\Big(\tau(\alpha_{k})\alpha_{k-\tau(\alpha_{k})}\beta_{k}+\beta_{k}^2\Big)\notag\\
&\leq (1-\rho\beta_{k})\|\Yhat_{k}\|^2 + \frac{\beta_{k}}{4\rho}\Eset\left[\|\Xhat_{k}\|^2\right] \notag\\ 
&\quad + \frac{7(2\sigma_{n}+1)^2(8\lambda_{1}+1)^{5}}{\lambda_{1}^{5}}\Big(\tau(\alpha_{k})\alpha_{k-\tau(\alpha_{k})}\beta_{k} + \beta_{k}^2\Big)\Eset\left[\|\Zhat_{k}\|^2\right]\notag\\
&\quad + \frac{13(4\sigma_{n}+1)(8\lambda_{1}+1)^5(2B+\|\Ybf^*\|)^2}{
\lambda_{1}^7}\Big(\tau(\alpha_{k})\alpha_{k-\tau(\alpha_{k})}\beta_{k}+\beta_{k}^2\Big).
\end{align*}
\end{proof}

\subsection{Proof of Lemma \ref{lem:Zhat2_bounded}}

\begin{proof}
Since $\alpha_{k}$ and $\beta_{k}$ satisfy \eqref{thm_rate:stepsizes} we first have
\begin{align*}
(1-\gamma\alpha_{k})\Eset\left[\|\Xhat_{k}\|^2\right] + (1-\rho\beta_{k})\Eset\left[\|\Yhat_{k}\|^2\right] + \frac{\beta_{k}}{4\rho}\Eset\left[\|\Xhat_{k}\|^2\right]
\leq \Eset\left[\|\Zhat_{k}\|^2\right].
\end{align*}
Adding Eq.\ \eqref{lem_Xhat2:Ineq} to Eq.\ \eqref{lem_Yhat2:Ineq} and using the preceding relation and  
\begin{align}
\Eset\left[\|\Zhat_{k+1}\|^2\right] & \leq (1-\gamma\alpha_{k})\Eset\left[\|\Xhat_{k}\|^2\right] + \frac{\alpha_{k}^2}{16}\Eset\left[\|\Xhat_{k}\|^2\right] + \frac{(1+\alpha_{0})^2(1+\sigma_{n})^2\beta_{k}^2}{\gamma\lambda_{1}^2\alpha_{k}}\Eset\left[\|\Zhat_{k}\|^2\right]\notag\\
&\quad + (1-\rho\beta_{k})\Eset\left[\|\Yhat_{k}\|^2\right] + \frac{\beta_{k}}{4\rho}\Eset\left[\|\Xhat_{k}\|^2\right]\notag\\
&\quad + \frac{6(1+\sigma_{n})^2(8\lambda_{1}+1)^{5}}{\lambda_{1}^{6}}\left(\tau(\alpha_{k})\alpha_{k-\tau(\alpha_{k})}\alpha_{k} + \alpha_{k}^2+\alpha_{k}\beta_{k}\right)\Eset\left[\|\Zhat_{k}\|^2\right]\notag\\
&\quad + \frac{6(1+4\sigma_{n})(1+8\lambda_{1})^5(2B+\|Y^*\|)^2}{\lambda_{1}^8}\left(\tau(\alpha_{k})\alpha_{k-\tau(\alpha_{k})}\alpha_{k}+\alpha_{k}^2\right)\notag\\
&\quad + \frac{7(2\sigma_{n}+1)^2(8\lambda_{1}+1)^{5}}{\lambda_{1}^{5}}\Big(\tau(\alpha_{k})\alpha_{k-\tau(\alpha_{k})}\beta_{k} + \beta_{k}^2\Big)\Eset\left[\|\Zhat_{k}\|^2\right]\notag\\
&\quad + \frac{13(4\sigma_{n}+1)(8\lambda_{1}+1)^5(2B+\|\Ybf^*\|)^2}{
\lambda_{1}^7}\Big(\tau(\alpha_{k})\alpha_{k-\tau(\alpha_{k})}\beta_{k}+\beta_{k}^2\Big)\notag\\
&\leq \Eset\left[\|\Zhat_{k}\|^2\right] + \Gamma_{1}\left(\tau(\alpha_{k})\alpha_{k-\tau(\alpha_{k})}\alpha_{k} + \beta_{k}^2 + \alpha_{k}^2+\frac{\beta_{k}^2}{\alpha_{k}}\right)\Eset\left[\|\Zhat_{k}\|^2\right]\notag\\
&\quad + \Gamma_{2}\left(\tau(\alpha_{k})\alpha_{k-\tau(\alpha_{k})}\alpha_{k} + \beta_{k}^2 +\alpha_{k}^2\right),\label{lem_Zhat2_bounded:Eq1}
\end{align}
where we use Assumption \ref{assump:bounded} to have $\sigma_{n}\leq 1/4$, and $\Gamma_{1}$ and $\Gamma_{2}$ are defined as 
\begin{align*}
&\Gamma_{1} =  \frac{30(\gamma+1)(8\lambda_{1}+1)^{5}(1+\alpha_{0})^2}{\gamma\lambda_{1}^{6}}\\
&\Gamma_2 = \frac{38(1+8\lambda_{1})^5(2B+\|Y^*\|)^2}{\lambda_{1}^8}\cdot
\end{align*}
Let $w_{k}$ satisfy $w_{0} = 1$ and
\begin{align}
w_{k} = \prod_{t=0}^{k}\left(1+\Gamma_{1}\left(\tau(\alpha_{t})\alpha_{t-\tau(\alpha_{t})}\alpha_{t} + \beta_{t}^2 + \alpha_{t}^2+\frac{\beta_{t}^2}{\alpha_{t}}\right)\right),\label{notation:wk}   
\end{align}
On the one hand, using $(1+x)\leq e^{x}$ for all $x\geq 0$ and \eqref{thm_rate:stepsizes} we have
\begin{align}
w_{k} &\leq e^{\sum_{t=0}^{k}\Gamma_{1}\left(\tau(\alpha_{t})\alpha_{t-\tau(\alpha_{t})}\alpha_{t} + \beta_{t}^2 + \alpha_{k}^2+\frac{\beta_{t}^2}{\alpha_{t}}\right)}\leq e^{C_{0}\Gamma_{1}}.  \label{notation:wk_upperbound}
\end{align}
On the other hand, using $1+x\geq e^{-x}$ for all $x\geq 0$ and \eqref{thm_rate:stepsizes} we obtain
\begin{align}
w_{k} &\geq e^{-\sum_{t=0}^{k}\Gamma_{1}\left(\tau(\alpha_{t})\alpha_{t-\tau(\alpha_{t})}\alpha_{t} + \beta_{t}^2 + \alpha_{k}^2+\frac{\beta_{t}^2}{\alpha_{t}}\right)}\geq e^{-C_{0}\Gamma_{1}}.  \label{notation:wk_lowerbound}
\end{align}
Thus, dividing both sides of Eq.\ \eqref{lem_Zhat2_bounded:Eq1} by $w_{k+1}$ and using Eq\ \eqref{notation:wk_lowerbound} give
\begin{align*}
\frac{\Eset\left[\|\Zhat_{k+1}\|^2\right]}{w_{k+1}} &\leq \frac{\Eset\left[\|\Zhat_{k}\|^2\right]}{w_{k}} + \frac{\Gamma_{2}}{e^{-C_{0}\Gamma_{1}}}\left(\tau(\alpha_{k})\alpha_{k-\tau(\alpha_{k})}\alpha_{k} + \beta_{k}^2 +\alpha_{k}^2\right)\notag\\
&\leq \Eset[\|\Zhat_{0}\|^2]+\frac{\Gamma_{2}}{e^{-C_{0}\Gamma_{1}}}\sum_{t=0}^{k}\left(\tau(\alpha_{t})\alpha_{t-\tau(\alpha_{t})}\alpha_{t} + \beta_{t}^2 +\alpha_{t}^2\right) \stackrel{\eqref{thm_rate:stepsizes}}{\leq} \Eset[\|\Zhat_{0}\|^2]+ \frac{C_{0}\Gamma_{2}}{e^{-C_{0}\Gamma_{1}}},
\end{align*}
which by using Eq.\ \eqref{notation:wk_upperbound} immediately gives Eq.\ \eqref{lem_Zhat2_bounded:Ineq}.
\end{proof}

\section{Conclusion}
In this paper, we studied a finite-time performance of the linear two-time-scale {\sf SA} under time-varying step sizes and Markovian noise. We show that the mean square errors of the variables generated by the method converge to zero at a sublinear rate $\Ocal(k^{2/3})$. In addition, we consider a restarting scheme to improve the performance of this method, in particular, in speeding up the transient time of the linear two-time-scale {\sf SA}. Few more interesting questions left from this work are the finite-time performance of the nonlinear counterparts and their applications in studying reinforcement learning algorithms with nonlinear function approximation.    

\bibliographystyle{IEEEtran}
\bibliography{refs}

\appendix

\section{Proof of Lemmas \ref{lem:XY_noise}}
We provide here the proof of Lemmas \ref{lem:XY_noise} stated in Section \ref{subsec:thm_analysis}. Recall that we denote by $\gamma > o$ and $\rho > 0$ the smallest eigenvalues of $\Abf_{11}$ and $\Delta$, respectively. In addition,  $\lambda_{1}\leq\ldots\leq \lambda_{n}$ are the singular values of $\Abf_{11}$ and $ \sigma_{1}\leq\ldots\leq \sigma_{n}$ are the singular values of $\Delta$. By Assumption \ref{assump:bounded}, we have $\lambda_{i}\leq 1/4$. Finally, for convenience we introduce the following notation
\begin{align}
&Z_{k} = \left[\begin{array}{cc}
X_{k}\\
Y_{k}
\end{array}\right],\quad \Zhat_{k} = \left[\begin{array}{cc}
\Xhat_{k}\\
\Yhat_{k}
\end{array}\right]\notag\\
&\tilde{\Abf}_{k}(\xi_{k}) = \left[\begin{array}{cc}
\Abf_{11}(\xi_{k}) & \Abf_{12}(\xi_{k})\\
\frac{\beta_{k}}{\alpha_{k}}\Abf_{21}(\xi_{k}) & \frac{\beta_{k}}{\alpha_{k}}\Abf_{22}(\xi_{k})
\end{array}\right],\quad \tilde{b}_{k}(\xi_{k}) =  \left[\begin{array}{cc}
b_{1}(\xi_{k})\\
\frac{\beta_{k}}{\alpha_{k}}b_{2}(\xi_{k})
\end{array}\right].\label{apx:notation} 
\end{align}
\subsection{Preliminaries}
We start by considering the following sequence of lemmas, which will be used later. We first provide some useful bounds for $\|Z_{k}-Z_{k-\tau(\alpha_{k})}\|$. 
\begin{lem}\label{lem:XY_bound}
Suppose that Assumption \ref{assump:bounded} holds. Let $\{\alpha_{k},\beta_{k}\}$ be two sequence of nonnegative and nonincreasing step sizes and $\Kcal^*$ satisfy \eqref{notation:K1*}. Then for all $k\geq \Kcal_{1}^*$ 
\begin{align}
&\|Z_{k}-Z_{k-\tau(\alpha_{k})}\| \leq 2\alpha_{k;\tau(\alpha_{k})}\|Z_{k-\tau(\alpha_{k})}\| + 4B\alpha_{k;\tau(\alpha_{k})}\label{lem_XY_bound:Ineq1a}.\\
&\|Z_{k}-Z_{k-\tau(\alpha_{k})}\| \leq 6\alpha_{k;\tau(\alpha_{k})}\|Z_{k}\| + 12B\alpha_{k;\tau(\alpha_{k})}.\label{lem_XY_bound:Ineq1b}
\end{align}
\end{lem}
\begin{proof}
Using $\tilde{\Abf}_{k}$ and $\tilde{b}_{k}$ in Eq.\ \eqref{apx:notation} , and by Eq.\ \eqref{alg:XY} we have
\begin{align}
Z_{k+1} = Z_{k} - \alpha_{k}\tilde{\Abf}_{k}(\xi_k)Z_{k} - \alpha_{k}\tilde{b}_{k}(\xi_{k}).\label{lem_XY_bound:Eq1}
\end{align}
Taking the $2-$norm on both sides of Eq.\ \eqref{lem_XY_bound:Eq1}  yields
\begin{align*}
\|Z_{k+1}\| &\leq \|Z_{k}\| + \alpha_{k}\|\tilde{\Abf}_{k}(\xi_{k})\|\|Z_{k}\| + \alpha_{k}\|\tilde{b}_{k}(\xi_{k})\|\notag\\
&\leq \|Z_{k}\| + \alpha_{k}\left(\|b_{1}(\xi_{k})\| + \frac{\beta_{k}}{\alpha_{k}}\|b_{2}(\xi_{k})\|\right) \notag\\
&\qquad + \alpha_{k}\left(\|\Abf_{11}(\xi_{k}) \| + \|\Abf_{12}(\xi_{k})\| + \frac{\beta_{k}}{\alpha_{k}}\|\Abf_{21}(\xi_{k})\| + \frac{\beta_{k}}{\alpha_{k}}\|\Abf_{22}(\xi_{k})\|\right)\|Z_{k}\|\notag\\
&\leq \|Z_{k}\| + \alpha_{k}\|Z_{k}\| + 2B\alpha_{k},
\end{align*}
where the last inequality is due to Assumption \ref{assump:bounded} and the fact that $\beta_{k}/\alpha_{k}\leq 1$. Using the preceding relation and by Eq.\ \eqref{notation:K1*} we have for all $k\geq \Kcal_{1}^*$ and $t\in[k-\tau(\alpha_{k}),k]$
\begin{align}
\|Z_{t}\| &\leq \left(1+\alpha_{t}\right)\|Z_{t}\| + 2B\alpha_{t}\notag\\
&\leq \prod_{\ell=k-\tau(\alpha_{k})}^{t}(1 + \alpha_{\ell})\|Z_{k-\tau(\alpha_{k})}\| + 2B\sum_{\ell=k-\tau(\alpha_{k})}^{t}\alpha_{t}\prod_{u=\ell+1}^{t}(1+\alpha_{u})\notag\\
&\leq \|Z_{k-\tau(\alpha_{k})}\|\exp\left\{\sum_{\ell=k-\tau(\alpha_{k})}^{t}\alpha_{\ell}\right\} + 2B\sum_{\ell=k-\tau(\alpha_{k})}^{t}\alpha_{t}\exp\left\{\sum_{u=\ell+1}^{k}\alpha_{u}\right\}\notag\\
&\leq 2\|Z_{k-\tau(\alpha_{k})}\| + 4B\sum_{\ell=k-\tau(\alpha_{k})}^{t}\alpha_{t}\label{lem_XY_bound:Eq1a},
\end{align}
where the third inequality we use the relation $(1+x)\leq e^{x}$ for all $x\geq0$. Next, by the triangle inequality and using Eq.\ \eqref{lem_XY_bound:Eq1} we obtain Eq.\ \eqref{lem_XY_bound:Ineq1a}, i.e., for all $k\geq \Kcal_{1}^*$
\begin{align*}
\|Z_{k}-Z_{k-\tau(\alpha_{k})}\|&\leq \sum_{t=k-\tau(\alpha_{k})}^{k-1}\|Z_{t+1}-Z_{t}\|\notag\\
&\leq \sum_{t=k-\tau(\alpha_{k})}^{k-1}\alpha_{t}\|\tilde{\Abf}_{t}(\xi_{t})Z_{t}\| + \sum_{t=k-\tau(\alpha_{k})}^{k-1}\alpha_{t}\|\tilde{b}_{t}(\xi_{t})\|\notag\\
&\leq \sum_{t=k-\tau(\alpha_{k})}^{k-1}\alpha_{t}\|Z_{t}\| + 2B \sum_{t=k-\tau(\alpha_{k})}^{k-1}\alpha_{t}\\
&\leq \sum_{t=k-\tau(\alpha_{k})}^{k-1}\alpha_{t}\left(2\|Z_{k-\tau(\alpha_{k})}\| + 4B\sum_{\ell=k-\tau(\alpha_{k})}^{t}\alpha_{t}\right) + 2B\sum_{t=k-\tau(\alpha_{k})}^{k-1}\alpha_{t}\notag\\
&\leq 2\alpha_{k;\tau(\alpha_{k})}\|Z_{k-\tau(\alpha_{k})}\| + 4B\alpha_{k;\tau(\alpha_{k})},
\end{align*}
where the last inequality we use Eq.\ \eqref{notation:K1*} and $\log(2)\leq 1/2$.  Finally, using the triangle inequality the preceding relation yields
\begin{align*}
\|Z_{k}-Z_{k-\tau(\alpha_{k})}\| \leq 2\alpha_{k;\tau(\alpha_{k})}\|Z_{k} - Z_{k-\tau(\alpha_{k})}\| + 2\alpha_{k;\tau(\alpha_{k})}\|Z_{k}\| + 4B\alpha_{k;\tau(\alpha_{k})}, 
\end{align*}
which by using Eq.\ \eqref{notation:K1*} and $\log(2)\leq 1/3$ we obtain Eq.\ \eqref{lem_XY_bound:Ineq1b}, i.e., 
\begin{align*}
\|Z_{k}-Z_{k-\tau(\alpha_{k})}\| \leq 6\alpha_{k;\tau(\alpha_{k})}\|Z_{k}\| + 12B\alpha_{k;\tau(\alpha_{k})}.
\end{align*}
\end{proof}
Similarly, we obtain a sequence of upper bounds for $\|\Zhat_{k}-\Zhat_{k-\tau(\alpha_{k})}\|$. 
\begin{lem}\label{lem:XYhat_bound}
Let all the conditions in Lemma \ref{lem:XY_bound} hold. Then for all $k\geq \Kcal_{1}^*$
\begin{align}
&\|\Zhat_{k}-\Zhat_{k-\tau(\alpha_{k})}\| \leq \frac{(8\lambda_{1}+1)^2}{8\lambda_{1}^2}\alpha_{k;\tau(\alpha_k)}\|\Zhat_{k-\tau(\alpha_{k})}\| + \frac{(8\lambda_{1}+1)^2(2B+\|Y^*\|)}{8\lambda_{1}^3} \alpha_{k;
\tau(\alpha_{k})}\label{lem_XYhat_bound:Ineq1a}.\\
&\|\Zhat_{k}-\Zhat_{k-\tau(\alpha_{k})}\| \leq \frac{3(8\lambda_{1}+1)^2}{8\lambda_{1}^2}\alpha_{k;\tau(\alpha_k)}\|\Zhat_{k}\| + \frac{3(8\lambda_{1}+1)^2(2B+\|Y^*\|)}{8\lambda_{1}^3} \alpha_{k;
\tau(\alpha_{k})}.\label{lem_XYhat_bound:Ineq1b}\\
& \|\Zhat_{k}-\Zhat_{k-\tau(\alpha_{k})}\|^2 \leq \frac{9(8\lambda_{1}+1)^4}{32\lambda_{1}^4}\alpha_{k;\tau(\alpha_k)}^2\|\Zhat_{k}\|^2 +  \frac{9(8\lambda_{1}+1)^4(2B+\|Y^*\|)^2}{32\lambda_{1}^6} \alpha_{k;\tau(\alpha_{k})}^2.\label{lem_XYhat_bound:Ineq1c}
\end{align}
\end{lem}

\begin{proof}
Recall that $\Zhat_{k} = [\Xhat_{k}^T,\Yhat_{k}^T]^T$, and by Eq.\ \eqref{alg:XY_hat} we have
\begin{align*}
\Zhat_{k} =  \left[\begin{array}{cc}
\Ibf     &  \Abf_{11}^{-1}\Abf_{12}\\
0     & \Ibf
\end{array}\right]Z_{k} -  \left[\begin{array}{c}
\Abf_{11}^{-1}b_{1}\\
Y^*
\end{array}\right],
\end{align*}
which implies that
\begin{align*}
Z_{k} =  \left[\begin{array}{cc}
\Ibf     &  -\Abf_{11}^{-1}\Abf_{12}\\
0     & \Ibf
\end{array}\right]\left(\Zhat_{k} +  \left[\begin{array}{c}
\Abf_{11}^{-1}b_{1}\\
Y^*
\end{array}\right]\right),
\end{align*}
Note that $\lambda_{i} \leq \ldots \leq \lambda_{n}$ are the singular values of $\Abf_{11}$ implying that $1/\lambda_1\geq\ldots\geq 1/\lambda_{n}$ are the ones of $\Abf_{11}^{-1}$.  Thus, using Assumption \ref{assump:bounded} the preceding relation gives
\begin{align}
\|Z_{k}\| &\leq \left(2 + \frac{1}{4\lambda_{1}}\right)\|\Zhat_{k}\| + \frac{(8\lambda_{1}+1)(B+\|Y^*\|)}{4\lambda_{1}^2}\notag\\
&= \frac{8\lambda_{1}+1}{4\lambda_{1}}\|\Zhat_{k}\| + \frac{(8\lambda_{1}+1)(B+\|Y^*\|)}{4\lambda_{1}^2}\cdot\label{lem_XYhat_bound:Eq1a}
\end{align}
On the other hand, using Eq.\ \eqref{alg:XY_hat} one more time yields \begin{align*}
\Xhat_{k} - \Xhat_{k-\tau(\alpha_{k})}  &= X_{k} - X_{k-\tau(\alpha_{k})} + \Abf_{11}^{-1}\Abf_{12}(Y_{k}-Y_{k-\tau(\alpha_{k})})\notag\\
\Yhat_{k} - \Yhat_{k-\tau(\alpha_{k})}  &= Y_{k} - Y_{k-\tau(\alpha_{k})},
\end{align*}
which implies that 
\begin{align}
\Zhat_{k} - \Zhat_{k-\tau(\alpha_{k})} = \left[\begin{array}{cc}
\Ibf     &  \Abf_{11}^{-1}\Abf_{12}\\
0     & \Ibf
\end{array}\right] (Z_{k}-Z_{k-\tau(\alpha_{k})}).\label{lem_XYhat_bound:Eq1c}
\end{align}
Thus, we have
\begin{align}
\|\Zhat_{k} - \Zhat_{k-\tau(\alpha_{k})}\|&\leq \left(2 + \frac{1}{4\lambda_{1}}\right)\|Z_{k}-Z_{k-\tau(\alpha_{k})}\| = \frac{8\lambda_{1}+1}{4\lambda_{1}}\|Z_{k}-Z_{k-\tau(\alpha_{k})}\|,\label{lem_XYhat_bound:Eq1b}
\end{align}
which by using Eqs.\ \eqref{lem_XY_bound:Ineq1a} and \eqref{lem_XYhat_bound:Eq1a} yields Eq.\ \eqref{lem_XYhat_bound:Ineq1a}, i.e.,
\begin{align*}
&\|\Zhat_{k} - \Zhat_{k-\tau(\alpha_{k})}\|
\stackrel{\eqref{lem_XY_bound:Ineq1a}}{\leq} \frac{(8\lambda_{1}+1)}{2\lambda_{1}}\alpha_{k;\tau(\alpha_k)}\|Z_{k-\tau(\alpha_{k})}\| + \frac{B(8\lambda_{1}+1)}{\lambda_{1}  } \alpha_{k;
\tau(\alpha_{k})}\notag\\
&\stackrel{\eqref{lem_XYhat_bound:Eq1a}}{\leq} \frac{(8\lambda_{1}+1)^2}{8\lambda_{1}^2}\alpha_{k;\tau(\alpha_k)}\|\Zhat_{k-\tau(\alpha_{k})}\| + \frac{8\lambda_{1}+1}{\lambda_{1}  } \left(B+\frac{(8\lambda_{1}+1)(B+\|Y^*\|)}{8\lambda_{1}^2}\right)\alpha_{k;
\tau(\alpha_{k})}\\
&\leq \frac{(8\lambda_{1}+1)^2}{8\lambda_{1}^2}\alpha_{k;\tau(\alpha_k)}\|\Zhat_{k-\tau(\alpha_{k})}\| + \frac{(8\lambda_{1}+1)^2(2B+\|Y^*\|)}{8\lambda_{1}^3} \alpha_{k;
\tau(\alpha_{k})}.
\end{align*}
Similarly, using Eq.\ \eqref{lem_XY_bound:Ineq1b} into Eq.\ \eqref{lem_XYhat_bound:Eq1b} gives Eq.\ \eqref{lem_XYhat_bound:Ineq1b}, i.e.,
\begin{align*}
\|\Zhat_{k} - \Zhat_{k-\tau(\alpha_{k})}\|&\leq \frac{3(8\lambda_{1}+1)}{2\lambda_{1}}\alpha_{k;\tau(\alpha_{k})}\|Z_{k}\| + \frac{3B(8\lambda_{1}+1)}{\lambda_{1}}\alpha_{k;\tau(\alpha_{k})}\notag\\
&\leq \frac{3(8\lambda_{1}+1)^2}{8\lambda_{1}^2}\alpha_{k;\tau(\alpha_k)}\|\Zhat_{k}\| + \frac{3(8\lambda_{1}+1)^2(2B+\|Y^*\|)}{8\lambda_{1}^3} \alpha_{k;
\tau(\alpha_{k})}.
\end{align*}
Finally, using the relation $(x+y)^2\leq 2x^2 + 2y^2$ we obtain Eq.\ \eqref{lem_XYhat_bound:Ineq1c}, i.e.,
\begin{align*}
\|\Zhat_{k} - \Zhat_{k-\tau(\alpha_{k})}\|^2 \leq \frac{9(8\lambda_{1}+1)^4}{32\lambda_{1}^4}\alpha_{k;\tau(\alpha_k)}^2\|\Zhat_{k}\|^2 +  \frac{9(8\lambda_{1}+1)^4(2B+\|Y^*\|)^2}{32\lambda_{1}^6} \alpha_{k;\tau(\alpha_{k})}^2.
\end{align*}
\end{proof}

\begin{lem}\label{lem:Ztau_bound}
Let all the conditions in Lemma \ref{lem:XY_bound} hold. Then for all $k\geq \Kcal_{1}^*$  
\begin{align}
& \|\Zhat_{k-\tau(\alpha_k)}\|\|Z_{k-\tau(\alpha_{k})}\|\leq \frac{(8\lambda_{1}+1)^{5}}{16\lambda_{1}^{5}}\|\Zhat_{k}\|^2 + \frac{(8\lambda_{1}+1)^5(2B+\|Y^*\|)^2}{16\lambda_{1}^7}\cdot \label{lem_Ztau_bound:Ineq1a}\\
&\|\Zhat_{k-\tau(\alpha_{k})}\|\leq \frac{(8\lambda_{1}+1)^4}{32\lambda_{1}^4}\|\Zhat_{k}\|^2 +  \frac{(8\lambda_{1}+1)^4(2B+\|Y^*\|)^2}{32\lambda_{1}^6}.\label{lem_Ztau_bound:Ineq1b}
\end{align}
\end{lem}

\begin{proof}
Using Eq.\ \eqref{lem_XYhat_bound:Eq1a} we first consider
\begin{align}
& \|\Zhat_{k-\tau(\alpha_k)}\|\|Z_{k-\tau(\alpha_{k})}\|\leq \|\Zhat_{k-\tau(\alpha_k)}\|\left(\frac{8\lambda_{1}+1}{4\lambda_{1}}\|\Zhat_{k-\tau(\alpha_{k})}\| + \frac{(8\lambda_{1}+1)(B+\|Y^*\|)}{4\lambda_{1}^2}\right)\notag\\
&\leq \frac{8\lambda_{1}+1}{4\lambda_{1}}\|\Zhat_{k-\tau(\alpha_{k})}\|^2 + \frac{(8\lambda_{1}+1)(B + \|Y^*\|)}{4\lambda_{1}^2}\|\Zhat_{k-\tau(\alpha_{k})}\|\notag\\
&\leq \frac{8\lambda_{1}+1}{4\lambda_{1}}\|\Zhat_{k-\tau(\alpha_{k})}\|^2 +  \frac{8\lambda_{1}+1}{8\lambda_{1}}\|\Zhat_{k-\tau(\alpha_{k})}\|^2 + \frac{(8\lambda_{1}+1)(B+\|Y^*\|)^2}{8\lambda_{1}^3}\notag\\
&= \frac{3(8\lambda_{1}+1)}{8\lambda_{1}}\|\Zhat_{k-\tau(\alpha_{k})}\|^2 +  \frac{(8\lambda_{1}+1)(B+\|Y^*\|)^2}{8\lambda_{1}^3}\notag\\
&\leq \frac{(8\lambda_{1}+1)}{\lambda_{1}}\|\Zhat_{k}-\Zhat_{k-\tau(\alpha_{k})}\|^2 +\frac{8\lambda_{1}+1}{\lambda_{1}}\|\Zhat_{k}\|^2 +  \frac{(8\lambda_{1}+1)(B+\|Y^*\|)^2}{8\lambda_{1}^3}\cdot\label{lem_Ztau_bound:Eq1a}
\end{align}
Next, by \eqref{notation:K1*} we have $\alpha_{k;\tau(\alpha_{k})}\leq \log(2)\leq 1/3$ for all $k\geq\Kcal_{1}^*$. Thus, using Eq.\ \eqref{lem_XYhat_bound:Ineq1c}  we have for all $k\geq\Kcal_{1}^*$
\begin{align*}
\|\Zhat_{k}-\Zhat_{k-\tau(\alpha_{k})}\|^2  &\leq \frac{9(8\lambda_{1}+1)^4}{32\lambda_{1}^4}\alpha_{k;\tau(\alpha_k)}^2\|\Zhat_{k}\|^2 +  \frac{9(8\lambda_{1}+1)^4(2B+\|Y^*\|)^2}{32\lambda_{1}^6} \alpha_{k;\tau(\alpha_{k})}^2\notag\\ 
&\leq \frac{(8\lambda_{1}+1)^4}{32\lambda_{1}^4}\|\Zhat_{k}\|^2 +  \frac{(8\lambda_{1}+1)^4(2B+\|Y^*\|)^2}{32\lambda_{1}^6} \cdot
\end{align*}
Substituting the preceding relation into the first term on the right-hand side of Eq.\ \eqref{lem_Ztau_bound:Eq1a} yields Eq.\ \eqref{lem_Ztau_bound:Ineq1a}
\begin{align*}
& \|\Zhat_{k-\tau(\alpha_k)}\|\|Z_{k-\tau(\alpha_{k})}\|\leq \frac{(8\lambda_{1}+1)^{5}}{16\lambda_{1}^{5}}\|\Zhat_{k}\|^2 + \frac{(8\lambda_{1}+1)^5(2B+\|Y^*\|)^2}{16\lambda_{1}^7}\cdot
\end{align*}
Similarly, we obtain Eq.\ \eqref{lem_Ztau_bound:Ineq1b}, i.e.,
\begin{align*}
&\|\Zhat_{k-\tau(\alpha_{k})}\|\leq \|\Zhat_{k} - \Zhat_{k-\tau(\alpha_{k})}\| + \|\Zhat_{k}\|\leq \frac{1}{2}\|\Zhat_{k} - \Zhat_{k-\tau(\alpha_{k})}\|^2 + \frac{1}{2}\|\Zhat_{k}\|^2 + 1\notag\\
&\leq \frac{1}{2}\left(\frac{(8\lambda_{1}+1)^4}{32\lambda_{1}^4}\|\Zhat_{k}\|^2 +  \frac{(8\lambda_{1}+1)^4(2B+\|Y^*\|)^2}{32\lambda_{1}^6} \right)+ \frac{1}{2}\|\Zhat_{k}\|^2 + 1\notag\\
& \leq \frac{(8\lambda_{1}+1)^4}{32\lambda_{1}^4}\|\Zhat_{k}\|^2 +  \frac{(8\lambda_{1}+1)^4(2B+\|Y^*\|)^2}{32\lambda_{1}^6}\cdot
\end{align*}
\end{proof}

\begin{lem}\label{lem:Zhat_bound}
Let all the conditions in Lemma \ref{lem:XY_bound} hold. Then for all $k\geq \Kcal_{1}^*$
\begin{align}
&\|\Zhat_{k-\tau(\alpha_{k})}\|\|Z_{k}-Z_{k-\tau(\alpha_{k})}\|\notag\\ &\qquad \leq \frac{3(8\lambda_{1}+1)^5}{16\lambda_{1}^5}\alpha_{k;
\tau(\alpha_{k})}\|\Zhat_{k}\|^2 + \frac{3(8\lambda_{1}+1)^5(2B+\|Y^*\|)^2}{32\lambda_{1}^7} \alpha_{k;
\tau(\alpha_{k})}.\label{lem_Zhat_bound:Ineq1a}\\
&\|\Zhat_{k}-\Zhat_{k-\tau(\alpha_{k})}\|\|Z_{k-\tau(\alpha_{k}})   \|\notag\\ 
&\qquad \leq \frac{3(8\lambda_{1}+1)^5}{8\lambda_{1}^5}\alpha_{k;\tau(\alpha_k)}\|\Zhat_{k}\|^2 + \frac{3(8\lambda_{1}+1)^5(2B+\|Y^*\|)^2}{8\lambda_{1}^7} \alpha_{k;\tau(\alpha_{k})}.\label{lem_Zhat_bound:Ineq1b}\\
&\|\Zhat_{k}-\Zhat_{k-\tau(\alpha_{k})}\|\|Z_{k}-Z_{k-\tau(\alpha_{k})}\| \notag\\ 
&\qquad  \leq  \frac{3(8\lambda_{1}+1)^5}{128\lambda_{1}^5}\alpha_{k;\tau(\alpha_k)}\|\Zhat_{k}\|^2 +  \frac{3(8\lambda_{1}+1)^5(2B+\|Y^*\|)^2}{128\lambda_{1}^7} \alpha_{k;\tau(\alpha_{k})}.\label{lem_Zhat_bound:Ineq1c}
\end{align}
\end{lem}
\begin{proof}
Using Eq.\ \eqref{lem_XYhat_bound:Eq1c} yields
\begin{align*}
Z_{k}-Z_{k-\tau(\alpha_{k})} = \left[\begin{array}{cc}
\Ibf     &  -\Abf_{11}^{-1}\Abf_{12}\\
0     & \Ibf
\end{array}\right](\Zhat_{k} - \Zhat_{k-\tau(\alpha_{k})}), 
\end{align*}
which by using $\lambda_{1}$ the smallest singular value of $\Abf_{11}$ and Assumption \ref{assump:bounded} gives 
\begin{align}
\|Z_{k}-Z_{k-\tau(\alpha_{k})}\| \leq \frac{8\lambda_{1}+1}{4\lambda_{1}}\|\Zhat_{k} - \Zhat_{k-\tau(\alpha_{k})}\|.\label{lem_Zhat_bound:Eq1a} 
\end{align}
Using the preceding relation we next consider
\begin{align*}
&\|\Zhat_{k-\tau(\alpha_{k})}\|\|Z_{k}-Z_{k-\tau(\alpha_{k})}\| \leq \frac{8\lambda_{1}+1}{4\lambda_{1}}     \|\Zhat_{k-\tau(\alpha_{k})}\| \|\Zhat_{k} - \Zhat_{k-\tau(\alpha_{k})}\|\notag\\
&\leq \frac{8\lambda_{1}+1}{4\lambda_{1}} \|\Zhat_{k}\|\|\Zhat_{k} - \Zhat_{k-\tau(\alpha_{k})}\| + \frac{8\lambda_{1}+1}{4\lambda_{1}}  \|\Zhat_{k} - \Zhat_{k-\tau(\alpha_{k})}\|^2,
\end{align*}
which by using Eqs.\ \eqref{lem_XYhat_bound:Ineq1b}, \eqref{lem_XYhat_bound:Ineq1c}, and \eqref{notation:K1*}  (to have $ \alpha_{k;\tau(\alpha_{k})}\leq 1/3$) yields
\begin{align*}
&\|\Zhat_{k-\tau(\alpha_{k})}\|\|Z_{k}-Z_{k-\tau(\alpha_{k})}\|\notag\\ 
&\leq \frac{8\lambda_{1}+1}{4\lambda_{1}}\|\Zhat_{k}\|\left(\frac{3(8\lambda_{1}+1)^2}{8\lambda_{1}^2}\alpha_{k;\tau(\alpha_k)}\|\Zhat_{k}\| + \frac{3(8\lambda_{1}+1)^2(2B+\|Y^*\|)}{8\lambda_{1}^3} \alpha_{k;
\tau(\alpha_{k})}\right)\notag\\ 
&\qquad + \frac{8\lambda_{1}+1}{4\lambda_{1}} \left(\frac{9(8\lambda_{1}+1)^4}{32\lambda_{1}^4}\alpha_{k;\tau(\alpha_k)}^2\|\Zhat_{k}\|^2 +  \frac{9(8\lambda_{1}+1)^4(2B+\|Y^*\|)^2}{32\lambda_{1}^6} \alpha_{k;\tau(\alpha_{k})}^2\right)\notag\\
&\leq \frac{3(8\lambda_{1}+1)^3}{32\lambda_{1}^3}\alpha_{k;
\tau(\alpha_{k})}\|\Zhat_{k}\|^2 + \frac{3(8\lambda_{1}+1)^3(2B+\|Y^*\|)}{32\lambda_{1}^4} \alpha_{k;
\tau(\alpha_{k})} \|\Zhat_{k}\| \notag\\
&\qquad + \frac{3(8\lambda_{1}+1)^5}{128\lambda_{1}^5}\alpha_{k;\tau(\alpha_k)}\|\Zhat_{k}\|^2 +  \frac{3(8\lambda_{1}+1)^5(2B+\|Y^*\|)^2}{128\lambda_{1}^7} \alpha_{k;\tau(\alpha_{k})}.
\end{align*}
Applying the inequality $2xy\leq x^2 + y^2$ $\forall x,y\in\Rset$ to the second term yields  Eq.\ \eqref{lem_Zhat_bound:Ineq1a}
\begin{align*}
&\|\Zhat_{k-\tau(\alpha_{k})}\|\|Z_{k}-Z_{k-\tau(\alpha_{k})}\|\notag\\
&\leq \frac{3(8\lambda_{1}+1)^3}{32\lambda_{1}^3}\alpha_{k;
\tau(\alpha_{k})}\|\Zhat_{k}\|^2 + \frac{3(8\lambda_{1}+1)^3}{64\lambda_{1}^3} \alpha_{k;
\tau(\alpha_{k})} \|\Zhat_{k}\|^2\notag\\ 
&\qquad  + \frac{3(8\lambda_{1}+1)^3(2B+\|Y^*\|)^2}{64\lambda_{1}^5} \alpha_{k;
\tau(\alpha_{k})} \notag\\ 
&\qquad + \frac{3(8\lambda_{1}+1)^5}{128\lambda_{1}^5}\alpha_{k;\tau(\alpha_k)}\|\Zhat_{k}\|^2 +  \frac{3(8\lambda_{1}+1)^5(2B+\|Y^*\|)^2}{128\lambda_{1}^7} \alpha_{k;\tau(\alpha_{k})}\notag\\
&\leq \frac{3(8\lambda_{1}+1)^5}{16\lambda_{1}^5}\alpha_{k;
\tau(\alpha_{k})}\|\Zhat_{k}\|^2 + \frac{3(8\lambda_{1}+1)^5(2B+\|Y^*\|)^2}{32\lambda_{1}^7} \alpha_{k;
\tau(\alpha_{k})}. 
\end{align*}
Second, using Eq.\ \eqref{lem_XYhat_bound:Eq1a} we consider
\begin{align}
&\|\Zhat_{k}-\Zhat_{k-\tau(\alpha_{k})}\|\|Z_{k-\tau_{\alpha_{k}}}\|\notag\\ 
&\leq \|\Zhat_{k}-\Zhat_{k-\tau(\alpha_{k})}\|\left( \frac{8\lambda_{1}+1}{4\lambda_{1}}\|\Zhat_{k-\tau(\alpha_{k})}\| + \frac{(8\lambda_{1}+1)(B+\|Y^*\|)}{4\lambda_{1}^2}\right)\notag\\   
& \leq \frac{8\lambda_{1}+1}{4\lambda_{1}}\|\Zhat_{k}-\Zhat_{k-\tau(\alpha_{k})}\| \left( \|\Zhat_{k}-\Zhat_{k-\tau(\alpha_{k})}\| + \|\Zhat_{k}\| + \frac{(B+\|Y^*\|)}{\lambda_{1}}\right)\notag\\
&\leq  \frac{8\lambda_{1}+1}{4\lambda_{1}}\|\Zhat_{k}-\Zhat_{k-\tau(\alpha_{k})}\|^2 + \frac{8\lambda_{1}+1}{4\lambda_{1}}\|\Zhat_{k}-\Zhat_{k-\tau(\alpha_{k})}\|\|\Zhat_{k}\|\allowdisplaybreaks\notag\\ 
&\qquad + \frac{(8\lambda_{1}+1)(B + \|Y^*\|)}{4\lambda_{1}^2}\|\Zhat_{k}-\Zhat_{k-\tau(\alpha_{k})}\|.\label{lem_Zhat_bound:Eq1}
\end{align}
We now analyze each term on the right-hand side of Eq.\ \eqref{lem_Zhat_bound:Eq1}. First, using Eqs.\ \eqref{lem_XYhat_bound:Ineq1c} and \eqref{notation:K1*} (to have $\alpha_{k;\tau(\alpha_{k})}\leq 1/3$) the first term can be upper bounded by 
\begin{align*}
&\frac{8\lambda_{1}+1}{4\lambda_{1}}\|\Zhat_{k}-\Zhat_{k-\tau(\alpha_{k})}\|^2\notag\\ 
& \leq   \frac{3(8\lambda_{1}+1)^5}{128\lambda_{1}^5}\alpha_{k;\tau(\alpha_k)}\|\Zhat_{k}\|^2 +  \frac{3(8\lambda_{1}+1)^5(2B+\|Y^*\|)^2}{128\lambda_{1}^7} \alpha_{k;\tau(\alpha_{k})}^2.
\end{align*}
Next, we consider the second term by using Eq.\ \eqref{lem_XYhat_bound:Ineq1b}
\begin{align*}
&\frac{8\lambda_{1}+1}{4\lambda_{1}}\|\Zhat_{k}-\Zhat_{k-\tau(\alpha_{k})}\|\|\Zhat_{k}\|\notag\\
&\leq\frac{8\lambda_{1}+1}{4\lambda_{1}}\left(\frac{3(8\lambda_{1}+1)^2}{8\lambda_{1}^2}\alpha_{k;\tau(\alpha_k)}\|\Zhat_{k}\| + \frac{3(8\lambda_{1}+1)^2(2B+\|Y^*\|)}{8\lambda_{1}^3} \alpha_{k;
\tau(\alpha_{k})}\right)\|\Zhat_{k}\|\notag\\
&\leq \frac{3(8\lambda_{1}+1)^3}{32\lambda_{1}^3}\alpha_{k;\tau(\alpha_k)}\|\Zhat_{k}\|^2 + \frac{3(8\lambda_{1}+1)^3(2B+\|Y^*\|)}{32\lambda_{1}^4} \alpha_{k;
\tau(\alpha_{k})}\|\Zhat_{k}\|\notag\\
&\leq \frac{3(8\lambda_{1}+1)^3}{32\lambda_{1}^3}\alpha_{k;\tau(\alpha_k)}\|\Zhat_{k}\|^2 + \frac{3(8\lambda_{1}+1)^3}{64\lambda_{1}^3} \alpha_{k;
\tau(\alpha_{k})}\|\Zhat_{k}\|^2\notag\\
&\qquad + \frac{3(8\lambda_{1}+1)^3(2B+\|Y^*\|)^2}{64\lambda_{1}^5}\alpha_{k;
\tau(\alpha_{k})}\notag\\
&\leq \frac{3(8\lambda_{1}+1)^3}{16\lambda_{1}^3}\alpha_{k;\tau(\alpha_k)}\|\Zhat_{k}\|^2 + \frac{3(8\lambda_{1}+1)^3(2B+\|Y^*\|)^2}{64\lambda_{1}^5} \alpha_{k;
\tau(\alpha_{k})}.
\end{align*}
In addition, using Eq.\ \eqref{lem_XYhat_bound:Ineq1b} again the third term is upper bounded by 
\begin{align*}
& \frac{(8\lambda_{1}+1)(B + \|Y^*\|)}{4\lambda_{1}^2}\|\Zhat_{k}-\Zhat_{k-\tau(\alpha_{k})}\|\notag\\
&\leq \frac{3(8\lambda_{1}+1)^3(2B+\|Y^*\|)}{32\lambda_{1}^4}\alpha_{k;\tau(\alpha_k)}\|\Zhat_{k}\| + \frac{3(8\lambda_{1}+1)^3(2B+\|Y^*\|)^2}{32\lambda_{1}^5} \alpha_{k;
\tau(\alpha_{k})}\notag\\
&\leq \frac{3(8\lambda_{1}+1)^3}{64\lambda_{1}^3}\alpha_{k;\tau(\alpha_k)}\|\Zhat_{k}\|^2 + \frac{3(8\lambda_{1}+1)^3(2B+\|Y^*\|)^2}{16\lambda_{1}^5} \alpha_{k;
\tau(\alpha_{k})}. 
\end{align*}
Thus, substituting the preceding three relations into Eq.\ \eqref{lem_Zhat_bound:Eq1} yields Eq.\ \eqref{lem_Zhat_bound:Ineq1b}, i.e.,
\begin{align*}
&\|\Zhat_{k}-\Zhat_{k-\tau(\alpha_{k})}\|\|Z_{k-\tau_{\alpha_{k}}}\|\notag\\
&\leq  \frac{3(8\lambda_{1}+1)^5}{8\lambda_{1}^5}\alpha_{k;\tau(\alpha_k)}\|\Zhat_{k}\|^2 + \frac{3(8\lambda_{1}+1)^5(2B+\|Y^*\|)^2}{8\lambda_{1}^7} \alpha_{k;\tau(\alpha_{k})}.
\end{align*}
Finally, using Eqs.\ \eqref{lem_Zhat_bound:Eq1a} and \eqref{lem_XYhat_bound:Ineq1c} we obtain Eq.\ \eqref{lem_Zhat_bound:Ineq1c}, i.e.,
\begin{align*}
&\|\Zhat_{k}-\Zhat_{k-\tau(\alpha_{k})}\|\|Z_{k}-Z_{k-\tau(\alpha_{k})}\| \leq   \frac{8\lambda_{1}+1}{4\lambda_{1}}\|\Zhat_{k}-\Zhat_{k-\tau(\alpha_{k})}\|^2\notag\\
&\quad\leq  \frac{9(8\lambda_{1}+1)^5}{128\lambda_{1}^5}\alpha_{k;\tau(\alpha_k)}^2\|\Zhat_{k}\|^2 +  \frac{9(8\lambda_{1}+1)^5(2B+\|Y^*\|)^2}{128\lambda_{1}^7} \alpha_{k;\tau(\alpha_{k})}^2\\
&\quad \leq  \frac{3(8\lambda_{1}+1)^5}{128\lambda_{1}^5}\alpha_{k;\tau(\alpha_k)}\|\Zhat_{k}\|^2 +  \frac{3(8\lambda_{1}+1)^5(2B+\|Y^*\|)^2}{128\lambda_{1}^7} \alpha_{k;\tau(\alpha_{k})},
\end{align*}
where the last inequality is due to Eq.\ \eqref{notation:K1*}, i.e., $\alpha_{k;\tau(\alpha_{}k)}\leq 1/3$.
\end{proof}

\subsection{Proofs of Lemma \ref{lem:XY_noise}}\label{apx:proof_lemma_XY_nois}
We now utilize Lemmas \ref{lem:Ztau_bound} and \ref{lem:Zhat_bound} to show the results stated in Lemma \ref{lem:XY_noise}. Due to the Markov samples, the noise is not a Martingale difference and it is dependent. To circumvent this, we properly use the geometric mixing time $\tau$ of the underlying Markov chain, which allows a systematic treatment of the Markovian ``noise''. In particular, by considering the conditional expectation  w.r.t $\Fcal_{k-\tau(\alpha_k)}$ instead of $\Fcal_{k}$, the noise dependence becomes quantifiably weak for observations $\tau(\alpha_{k})$ steps apart. This observation has been used in \cite{SrikantY2019_FiniteTD,GuptaSY2019_twoscale,DoanMR2019_DTD,ChenZDMC2019} to address Markov samples for other settings. We now proceed to present our analysis.

\begin{proof}
As will be seen shortly, Eqs.\ \eqref{lem_XY_noise:Ineq1a}--\eqref{lem_XY_noise:Ineq1c} can be derived by using the same steps. Indeed, one can show these results through studying  $\Eset[\epsilon_{k}^T\Gamma\Xhat_{k}]$ for some given constant matrix $\Gamma$. Then, by choosing $\Gamma$ properly we can obtain the desired results. We start by using the definition of $\epsilon_{k}$ in Eq.\ \eqref{analysis:noise} to have
\begin{align}
\epsilon_{k} &=    \Abf_{11}(\xi_{k})X_{k} + \Abf_{12}(\xi_{k})Y_{k} + b_1(\xi_{k})- \Abf_{11}X_{k} - \Abf_{12}Y_{k} - b_1\notag\\ 
&= 
\left[\begin{array}{c}
(\Abf_{11}(\xi_{k})-\Abf_{11})^T \\
(\Abf_{12}(\xi_{k})-\Abf_{12})^T 
\end{array}\right]^T\left[\begin{array}{c}
X_{k} \\
Y_{k} 
\end{array}\right] + b_{1}(\xi_{k})-b_{1} = \Dbf_{1}(\xi_{k})Z_{k} + b_{1}(\xi_{k})-b_{1},\label{lem_XY_noise:Eq0}    
\end{align}
where recall that $Z_{k} = [X_{k}^T,Y_{k}]^T$ and $\Dbf_{1}(\xi_{k})$ is defined as
\begin{align}
\Dbf_{1}(\xi_{k}) = [\Abf_{11}(\xi_{k})-\Abf_{11},\;\Abf_{12}(\xi_{k})-\Abf_{12}]. \label{notation:D1}  
\end{align}
The equation above gives
\begin{align}
\Xhat_{k}^T\Gamma^T\epsilon_{k} 
= \Xhat_{k}^{T}\Gamma^T\Dbf_{1}(\xi_{k})Z_{k} + \Xhat_{k}^{T}\Gamma^T(b_{1}(\xi_{k})-b_{1}).\label{lem_XY_noise:Eq1}
\end{align}
We first consider the first term on the right-hand side of Eq.\ \eqref{lem_XY_noise:Eq1} as 
\begin{align}
\Xhat_{k}^T
\Gamma^T\Dbf_{1}(\xi_{k})Z_{k}&= \Xhat_{k-\tau(\alpha_{k})}^T\Gamma^T\Dbf_{1}(\xi_{k})Z_{k-\tau(\alpha_{k})}  + \Xhat_{k-\tau(\alpha_{k})}^T\Gamma^T\Dbf_{1}(\xi_{k})(
Z_{k} - Z_{k-\tau(\alpha_{k})}) \notag\\
&\qquad + (\Xhat_{k} - \Xhat_{k-\tau(\alpha_{k})})^T\Gamma^T\Dbf_{1}(\xi_{k})
Z_{k-\tau(\alpha_{k})}\notag\\ 
&\qquad + (\Xhat_{k} - \Xhat_{k-\tau(\alpha_{k})})^T\Gamma^T\Dbf_{1}(\xi_{k})(
Z_{k} - Z_{k-\tau(\alpha_{k})}).\label{lem_XY_noise:Eq1a}
\end{align}
Next, to give an upper bound for the right-hand side of the preceding relation, we consider the following four relations. Recall that $\Fcal_{k}$ contains all the history generated by the method up to time $k$. 
\begin{enumerate}[leftmargin = 4.5mm]
\item  Taking the conditional  expectation of the first term on the right-hand side of \eqref{lem_XY_noise:Eq1a} w.r.t $\Fcal_{k-\tau(\alpha_{k})}$ and using \eqref{notation:D1} and Assumption \ref{assump:mix} yield
\begin{align*}
&\Eset\left[\Xhat_{k-\tau(\alpha_{k})}^T\Gamma^T\Dbf_{1}(\xi_{k})Z_{k-\tau(\alpha_{k})}\,|\,\Fcal_{k-\tau(\alpha_{k})}\right]\notag\\ 
&\quad = \Eset\left[\Xhat_{k-\tau(\alpha_{k})}^T\Gamma^T\Eset\left[[\Abf_{11}(\xi_{k})-\Abf_{11},\;\Abf_{12}(\xi_{k})-\Abf_{12}]\,|\,\Fcal_{k-\tau(\alpha_{k})}\right]Z_{k-\tau(\alpha_{k})}\right]\notag\\
&\quad \leq \|\Gamma\|\|\Xhat_{k-\tau(\alpha_{k})}\|\|\Eset\left[[\Abf_{11}(\xi_{k})-\Abf_{11},\;\Abf_{12}(\xi_{k})-\Abf_{12}]\,|\,\Fcal_{k-\tau(\alpha_{k})}\right]\|\|Z_{k-\tau(\alpha_{k})}\|\notag\\
&\quad\leq 2\|\Gamma\|\alpha_{k}\|\Xhat_{k-\tau(\alpha_{k})}\| \|Z_{k-\tau(\alpha_{k})}\|\notag\\
&\quad \leq\frac{(8\lambda_{1}+1)^{5}\|\Gamma\|}{8\lambda_{1}^{5}}\alpha_{k}\|\Zhat_{k}\|^2 + \frac{(8\lambda_{1}+1)^5(2B+\|Y^*\|)^2\|\Gamma\|}{8\lambda_{1}^7}\alpha_{k},
\end{align*}
where the last inequality is due to \eqref{lem_Ztau_bound:Ineq1a}.
\item Using the triangle inequality, Eq. \eqref{notation:D1}, and Assumption \ref{assump:bounded} we have
\begin{align*}
&\Xhat_{k-\tau(\alpha_{k})}^T\Gamma^T\Dbf_{1}(\xi_{k})(
Z_{k} - Z_{k-\tau(\alpha_{k})})\leq\|\Gamma\| \|\Xhat_{k-\tau(\alpha_{k})}\|\|\Dbf_{1}(\xi_{k})\|\|Z_{k} - Z_{k-\tau(\alpha_{k})}\|\notag\\
&\qquad \leq \|\Gamma\|\|\Xhat_{k-\tau(\alpha_{k})}\|\|Z_{k} - Z_{k-\tau(\alpha_{k})}\|\notag\\
&\qquad \leq \frac{3(8\lambda_{1}+1)^5\|\Gamma\|}{128\lambda_{1}^5}\alpha_{k;
\tau(\alpha_{k})}\|\Zhat_{k}\|^2 + \frac{3(8\lambda_{1}+1)^5(2B+\|Y^*\|)^2\|\Gamma\|}{128\lambda_{1}^7} \alpha_{k;
\tau(\alpha_{k})},
\end{align*}
where the last inequality is due to Eq.\ \eqref{lem_Zhat_bound:Ineq1a}. Here recall that $\lambda_{1}$ is the smallest singular value of $\Abf_{11}$.
\item Using Eq.\ \eqref{lem_Zhat_bound:Ineq1b} yields
\begin{align*}
&(\Xhat_{k} - \Xhat_{k-\tau(\alpha_{k})})^T\Gamma^T\Dbf_{1}(\xi_{k})
Z_{k-\tau(\alpha_{k})}\leq \|\Gamma\|\|\Xhat_{k} - \Xhat_{k-\tau(\alpha_{k})})\|\|
Z_{k-\tau(\alpha_{k})}\|\notag\\
&\qquad\leq \frac{3(8\lambda_{1}+1)^5\|\Gamma\|}{8\lambda_{1}^5}\alpha_{k;\tau(\alpha_k)}\|\Zhat_{k}\|^2 + \frac{3(8\lambda_{1}+1)^5(2B+\|Y^*\|)^2\|\Gamma\|}{8\lambda_{1}^7} \alpha_{k;\tau(\alpha_{k})}.
\end{align*}
\item Finally, by Eq.\ \eqref{lem_Zhat_bound:Ineq1c} we have
\begin{align*}
&(\Xhat_{k} - \Xhat_{k-\tau(\alpha_{k})})^T\Gamma^T\Dbf_{1}(\xi_{k})(
Z_{k} - Z_{k-\tau(\alpha_{k})}) \leq \|\Gamma\|\|\Xhat_{k} - \Xhat_{k-\tau(\alpha_{k})}\|\|
Z_{k} - Z_{k-\tau(\alpha_{k})}\|\notag\\
&\qquad\leq \frac{3(8\lambda_{1}+1)^5\|\Gamma\|}{128\lambda_{1}^5}\alpha_{k;\tau(\alpha_k)}\|\Zhat_{k}\|^2 +  \frac{3(8\lambda_{1}+1)^5(2B+\|Y^*\|)^2\|\Gamma\|}{128\lambda_{1}^7} \alpha_{k;\tau(\alpha_{k})}.
\end{align*}
\end{enumerate}
We next take the expectation on both sides of Eq.\ \eqref{lem_XY_noise:Eq1a} and use the four relations above to have
\begin{align}
\Eset\left[\Xhat_{k}^T\Gamma^T
\Dbf_{1}(\xi_{k})Z_{k}\right] &\leq \frac{(8\lambda_{1}+1)^{5}\|\Gamma\|}{8\lambda_{1}^{5}}\alpha_{k}\Eset[\|\Zhat_{k}\|^2 ] + \frac{(8\lambda_{1}+1)^5(2B+\|Y^*\|)^2\|\Gamma\|}{8\lambda_{1}^7}\alpha_{k}\notag\\
&\quad + \frac{3(8\lambda_{1}+1)^5\|\Gamma\|}{128\lambda_{1}^5}\alpha_{k;
\tau(\alpha_{k})}\Eset[\|\Zhat_{k}\|^2 ] + \frac{3(8\lambda_{1}+1)^5(2B+\|Y^*\|)^2\|\Gamma\|}{128\lambda_{1}^7} \alpha_{k;
\tau(\alpha_{k})}\notag\\
&\quad + \frac{3(8\lambda_{1}+1)^5\|\Gamma\|}{8\lambda_{1}^5}\alpha_{k;\tau(\alpha_k)}\Eset[\|\Zhat_{k}\|^2 ] + \frac{3(8\lambda_{1}+1)^5(2B+\|Y^*\|)^2\|\Gamma\|}{8\lambda_{1}^7} \alpha_{k;\tau(\alpha_{k})}\notag\\
&\quad + \frac{3(8\lambda_{1}+1)^5\|\Gamma\|}{128\lambda_{1}^5}\alpha_{k;\tau(\alpha_k)} \Eset[\|\Zhat_{k}\|^2 ] +  \frac{3(8\lambda_{1}+1)^5(2B+\|Y^*\|)^2\|\Gamma\|}{128\lambda_{1}^7} \alpha_{k;\tau(\alpha_{k})}\notag\\
&\leq \frac{(8\lambda_{1}+1)^{5}\|\Gamma\|}{8\lambda_{1}^{5}}\alpha_{k} \Eset[\|\Zhat_{k}\|^2 ] + \frac{(8\lambda_{1}+1)^5(2B+\|Y^*\|)^2\|\Gamma\|}{8\lambda_{1}^7}\alpha_{k}\notag\\
&\quad  + \frac{3(8\lambda_{1}+1)^5\|\Gamma\|}{4\lambda_{1}^5}\alpha_{k;\tau(\alpha_k)} \Eset[\|\Zhat_{k}\|^2 ] +  \frac{3(8\lambda_{1}+1)^5(2B+\|Y^*\|)^2\|\Gamma\|}{4\lambda_{1}^7} \alpha_{k;\tau(\alpha_{k})}.
\label{lem_XY_noise:Eq1b}    
\end{align}
Similarly, we consider the last term on the right-hand side in \eqref{lem_XY_noise:Eq1}
\begin{align}
\Xhat_{k}^{T}\Gamma^T(b_{1}(\xi_{k})-b_{1}) = \Xhat_{k-\tau(\alpha_{k})}^{T}\Gamma^T(b_{1}(\xi_{k})-b_{1}) + (\Xhat_{k}-\Xhat_{k-\tau(\alpha_{k})})^{T}\Gamma^T(b_{1}(\xi_{k})-b_{1}).\label{lem_XY_noise:Eq2}  
\end{align}
Taking the conditional expectation of the first term on the right-hand side of Eq.\ \eqref{lem_XY_noise:Eq2} yields
\begin{align*}
\Eset\left[\Xhat_{k-\tau(\alpha_{k})}^{T}\Gamma^T(b_{1}(\xi_{k})-b_{1})\,|\Fcal_{k-\tau(\alpha_{k})}\,\right] \leq \alpha_{k}\|\Gamma\|\|\Xhat_{k-\tau(\alpha_{k})}\|.   
\end{align*}
Taking the expectation on both sides of Eq.\ \eqref{lem_XY_noise:Eq2} and using Eqs.\ \eqref{lem_XYhat_bound:Ineq1b}, \eqref{lem_Ztau_bound:Ineq1b}, and the  preceding relation give
\begin{align}
&\Eset\left[\Xhat_{k}^{T}(b_{1}(\xi_{k})-b_{1})\right]\notag\\ 
&\quad= \Eset\left[\Xhat_{k-\tau(\alpha_{k})}^{T}\Gamma^T(b_{1}(\xi_{k})-b_{1}) + (\Xhat_{k}-\Xhat_{k-\tau(\alpha_{k})})^{T}\Gamma^T(b_{1}(\xi_{k})-b_{1})\right]\notag\\ 
&\quad\leq \frac{(8\lambda_{1}+1)^4\|\Gamma\|}{32\lambda_{1}^4}\alpha_{k}\Eset\left[\|\Zhat_{k}\|^2\right] +  \frac{(8\lambda_{1}+1)^4(2B+\|Y^*\|)^2\|\Gamma\|}{32\lambda_{1}^6}\alpha_{k}\notag\\
&\quad\quad  + \frac{6B(8\lambda_{1}+1)^2\|\Gamma\|}{8\lambda_{1}^2}\alpha_{k;\tau(\alpha_k)}\Eset\left[\|\Zhat_{k}\|\right] + \frac{6B(8\lambda_{1}+1)^2(2B+\|Y^*\|)\|\Gamma\|}{8\lambda_{1}^3} \alpha_{k;
\tau(\alpha_{k})}\notag\\
&\quad\leq \frac{(8\lambda_{1}+1)^4\|\Gamma\|}{32\lambda_{1}^4}\alpha_{k}\Eset\left[\|\Zhat_{k}\|^2\right] +  \frac{(8\lambda_{1}+1)^4(2B+\|Y^*\|)^2\|\Gamma\|}{32\lambda_{1}^6}\alpha_{k}\notag\\
&\quad\quad  + \frac{3(8\lambda_{1}+1)^4\|\Gamma\|}{32\lambda_{1}^4}\alpha_{k;\tau(\alpha_k)}\Eset\left[\|\Zhat_{k}\|^2\right] + 3B^2\|\Gamma\|\alpha_{k;
\tau(\alpha_{k})}\notag\\ 
&\quad\quad + \frac{6B(8\lambda_{1}+1)^2(2B+\|Y^*\|)\|\Gamma\|}{8\lambda_{1}^3} \alpha_{k;
\tau(\alpha_{k})}\notag\\
&\quad\leq \frac{(8\lambda_{1}+1)^4\|\Gamma\|}{32\lambda_{1}^4}\alpha_{k}\Eset[\|\Zhat_{k}\|^2 ] +  \frac{(8\lambda_{1}+1)^4(2B+\|Y^*\|)^2\|\Gamma\|}{32\lambda_{1}^6}\alpha_{k}\notag\\
&\quad\quad  + \frac{3(8\lambda_{1}+1)^4}{32\lambda_{1}^4\|\Gamma\|}\alpha_{k;\tau(\alpha_k)}\Eset\left[\|\Zhat_{k}\|^2\right]+  \frac{3(8\lambda_{1}+1)^2(2B+\|Y^*\|)^2\|\Gamma\|}{4\lambda_{1}^3} \alpha_{k;
\tau(\alpha_{k})}.\label{lem_XY_noise:Eq1c}
\end{align}
Thus, taking the expectation on both sides of Eq.\ \eqref{lem_XY_noise:Eq1} and using \eqref{lem_XY_noise:Eq1b} and \eqref{lem_XY_noise:Eq1c} yields
\begin{align*}
\Eset\left[\epsilon_{k}^{T}\Gamma\Xhat_{k}\right] &\leq \frac{(8\lambda_{1}+1)^{5}\|\Gamma\|}{8\lambda_{1}^{5}}\alpha_{k} \Eset\left[\|\Zhat_{k}\|^2\right] + \frac{(8\lambda_{1}+1)^5(2B+\|Y^*\|)^2\|\Gamma\|}{8\lambda_{1}^7}\alpha_{k}\notag\\ 
&\qquad + \frac{3(8\lambda_{1}+1)^5\|\Gamma\|}{4\lambda_{1}^5}\alpha_{k;\tau(\alpha_k)} \Eset\left[\|\Zhat_{k}\|^2\right]\notag\\ 
&\qquad +  \frac{3(8\lambda_{1}+1)^5(2B+\|Y^*\|)^2\|\Gamma\|}{4\lambda_{1}^7} \alpha_{k;\tau(\alpha_{k})}\notag\\ 
&\qquad +\frac{(8\lambda_{1}+1)^4\|\Gamma\|}{32\lambda_{1}^4}\alpha_{k}\Eset\left[\|\Zhat_{k}\|^2\right]+  \frac{(8\lambda_{1}+1)^4(2B+\|Y^*\|)^2\|\Gamma\|}{32\lambda_{1}^6}\alpha_{k}\notag\\ 
&\qquad + \frac{3(8\lambda_{1}+1)^4}{32\lambda_{1}^4\|\Gamma\|}\alpha_{k;\tau(\alpha_k)}\Eset\left[\|\Zhat_{k}\|^2\right]\notag\\ 
&\qquad +  \frac{3(8\lambda_{1}+1)^2(2B+\|Y^*\|)^2\|\Gamma\|}{4\lambda_{1}^3} \alpha_{k;
\tau(\alpha_{k})}\notag\\
&\leq \frac{3(8\lambda_{1}+1)^{5}\|\Gamma\|}{2\lambda_{1}^{5}}\tau(\alpha_{k})\alpha_{k-\tau(\alpha_{k})}\Eset\left[\|\Zhat_{k}\|^2\right]\notag\\ 
&\qquad + \frac{3(8\lambda_{1}+1)^5(2B+\|Y^*\|)^2\|\Gamma\|}{
\lambda_{1}^7}\tau(\alpha_{k})\alpha_{k-\tau(\alpha_{k})},
\end{align*}
where in the last inequality we use $\alpha_{k}\leq \alpha_{k-\tau(\alpha_{k})}$ and $\alpha_{k;\tau_{\alpha_{k}}}\leq \tau(\alpha_{k})\alpha_{k-\tau(\alpha_{k})}$.\\
By letting $\Gamma = \Ibf$  gives us Eq.\ \eqref{lem_XY_noise:Ineq1a}. Moreover, using a similar approach as above immediately gives us Eqs.\ \eqref{lem_XY_noise:Ineq1b} and \eqref{lem_XY_noise:Ineq1c}. First, similar to Eq.\ \eqref{lem_XY_noise:Eq0} one can write $\psi_{k}$ by using Eq.\ \eqref{analysis:noise} as
\begin{align*}
\psi_{k} &= \Abf_{21}(\xi_{k})X_{k} + \Abf_{22}(\xi_{k})Y_{k} + b_2(\xi_{k})- \Abf_{21}X_{k} - \Abf_{22}Y_{k} - b_2\notag\\ 
&= 
\left[\begin{array}{c}
(\Abf_{21}(\xi_{k})-\Abf_{21})^T \\
(\Abf_{22}(\xi_{k})-\Abf_{22})^T 
\end{array}\right]^T\left[\begin{array}{c}
X_{k} \\
Y_{k} 
\end{array}\right] + (b_{2}(\xi_{k})-b_{2}).
\end{align*}
Second, note that $\max\{\|\Xhat_{k}\|\,,\,\|\Yhat_{k}\|\} \leq \|\Zhat_{k}\|$. Thus, by repeating the same line of analysis and using Assumption \ref{assump:bounded} we obtain Eqs.\ \eqref{lem_XY_noise:Ineq1b} and \eqref{lem_XY_noise:Ineq1c}.
\end{proof}

\subsection{Proof of Lemma \ref{lem:V0}}\label{apx:Proof_lem_V0}

\begin{proof}
Let $\Zhat_{k} = [\Xhat_{k}^T,\Yhat_{k}^T]^T$. Recall from \eqref{alg:XY_hat} that
\begin{align*}
\Zhat_{k} =  \left[\begin{array}{cc}
\Ibf     &  \Abf_{11}^{-1}\Abf_{12}\\
0     & \Ibf
\end{array}\right]Z_{k} -  \left[\begin{array}{c}
\Abf_{11}^{-1}b_{1}\\
Y^*
\end{array}\right],
\end{align*}
which implies
\begin{align*}
Z_{k} =  \left[\begin{array}{cc}
\Ibf     &  -\Abf_{11}^{-1}\Abf_{12}\\
0     & \Ibf
\end{array}\right]\left(\Zhat_{k} +  \left[\begin{array}{c}
\Abf_{11}^{-1}b_{1}\\
Y^*
\end{array}\right]\right).
\end{align*}
Thus, using Assumption \ref{assump:bounded}, i.e., $\lambda_{1}\leq 1/4$, we have
\begin{align*}
\|\Zhat_{k}\| \leq \frac{8\lambda_{1}+1}{4\lambda_{1}}\|Z_{k}\| + \frac{B+\|Y^*\|}{\lambda_{1}}\leq \frac{1}{\lambda_{1}}  \|Z_{k}\| + \frac{B+\|Y^*\|}{\lambda_{1}},  
\end{align*}
which gives
\begin{align}
\|\Zhat_{k}\|^2 \leq \frac{2}{\lambda_{1}^2}  \|Z_{k}\|^2 + \frac{2(B+\|Y^*\|)^2}{\lambda_{1}^2}\cdot    \label{lem_V0:Eq1}
\end{align}
Using the first inequality in Eq.\ \eqref{lem_XY_bound:Eq1a} we have
\begin{align*}
\|Z_{k+1}\| &\leq (1+\alpha_{k})\|Z_{k}\| + 2B\alpha_{k} \leq \prod_{t=0}^{k}(1+\alpha_{t})\|Z_{0}\| + 2B\sum_{t=0}^{k}\alpha_{t}\prod_{\ell=t+1}^{k}(1+\alpha_{\ell})\notag\\
&\leq (1+\alpha_{0})^{k+1} \|Z_{0}\| + 2B(1+\alpha_{0})^{k}\sum_{t=0}^{k}\frac{\alpha_{0}}{(t+1)^{2/3}(1+\alpha_{0})^{t}}\notag\\
&\leq (1+\alpha_{0})^{k+1} \|Z_{0}\| + 2\alpha_{0}B(1+\alpha_{0})^{k}.
\end{align*}
On the other hand using the equation of $Z_{k}$ above we have
\begin{align*}
\|Z_{k}\| \leq \frac{8\lambda_{1} + 1}{4\lambda_{1}}\|\Zhat_{k}\| + \frac{(8\lambda_{1} + 1)(B+\|Y^*\|)}{4\lambda_{1}^2} \leq \frac{1}{\lambda_{1}}\|\Zhat_{k}\| +  \frac{B+\|Y^*\|}{\lambda_{1}^2},
\end{align*}
which yields
\begin{align*}
\|Z_{k}\|^2 \leq     \frac{2}{\lambda_{1}^2}\|\Zhat_{k}\|^2 +  \frac{2(B+\|Y^*\|)^2}{\lambda_{1}^4}. 
\end{align*}
Thus, we obtain 
\begin{align*}
\|Z_{k+1}\|^2 &\leq 2(1+\alpha_{0})^{2(k+1)} \|Z_{0}\|^{2} + 8\alpha_{0}^2B^2(1+\alpha_{0})^{2k}\notag\\  
&\leq \frac{4(1+\alpha_{0})^{2(k+1)}}{\lambda_{1}^2}\|\Zhat_{0}\|^2 +  \frac{4(B+\|Y^*\|)^2}{\lambda_{1}^4}(1+\alpha_{0})^{2(k+1)} + 8\alpha_{0}^2B^2(1+\alpha_{0})^{2k}\notag\\
&\leq \frac{4(1+\alpha_{0})^{2(k+1)}}{\lambda_{1}^2}\|\Zhat_{0}\|^2 +  \frac{12(B+\|Y^*\|)^2}{\lambda_{1}^4}(1+\alpha_{0})^{2(k+1)}.
\end{align*}
In addition, using Eq.\ \eqref{thm_rate:stepsizes} we also have
\begin{align*}
\frac{1}{2\gamma\rho}\frac{\beta_{k}}{\alpha_{k}} \leq \frac{1}{2\gamma\rho}\frac{\beta_{0}}{\alpha_{0}}\leq 1.    
\end{align*}
Thus, using the preceding two relations and Eq.\ \eqref{lem_V0:Eq1} we consider
\begin{align*}
V_{k} &= \Eset\left[\|\Yhat_{k}\|^2\right] + \frac{1}{2\gamma\rho}\frac{\beta_{k}}{\alpha_{k}}\Eset\left[\|\Xhat_{k}\|^2\right]\leq \Eset\left[\|\Zhat_{k}\|^2\right] \leq \frac{2}{\lambda_{1}^2}\Eset\left[\|Z_{k}\|^2\right] + \frac{2(B+\|Y^*\|)^2}{\lambda_{1}^2} \notag\\
&\leq \frac{8(1+\alpha_{0})^{2k}}{\lambda_{1}^4}\Eset\left[\|\Zhat_{0}\|^2\right] +  \frac{25(B+\|Y^*\|)^2}{\lambda_{1}^6}(1+\alpha_{0})^{2k}\notag\\
&\leq \frac{8(\beta_{0}+\gamma\rho\alpha_{0})(1+\alpha_{0})^{2k}}{\beta_{0}\lambda_{1}^2}V_{0} +  \frac{25(B+\|Y^*\|)^2}{\lambda_{1}^6}(1+\alpha_{0})^{2k}.
\end{align*}
By letting $k=\Kcal^*$ we obtain Eq.\ \eqref{lem_V0:Ineq}
\end{proof}

\end{document}